\documentclass{article} % For LaTeX2e
\usepackage{iclr2027_conference,times}
\renewcommand{\headrulewidth}{0pt}
\usepackage{placeins}
\usepackage{amsmath,amsfonts,bm}
\usepackage{amsthm} %%
\usepackage{mathtools}
\usepackage{graphicx} %%
\usepackage{wrapfig} %%
\usepackage{subfigure} %%
\usepackage{caption}
\usepackage{booktabs}       % professional-quality tables
\usepackage{tikz}
\usetikzlibrary{matrix,arrows,decorations.pathmorphing}

\def\1{\bm{1}}

\newcommand{\R}{\mathbb{R}}

\newcommand{\eat}[1]{}

\newtheorem{theorem}{Theorem}[section]
\newtheorem{proposition}[theorem]{Proposition}
\newtheorem{corollary}[theorem]{Corollary}

\makeatletter

\newcommand{\Rmnum}[1]{\expandafter\@slowromancap\romannumeral #1@}
\makeatother

\newcommand{\GL}{\mathrm{GL}}

\newcommand{\Par}{\Theta}
\newcommand{\Data}{\mathcal{D}}

\newcommand{\id}{\text{\rm id}}

\theoremstyle{plain}
\newtheorem{manualtheoreminner}{Theorem}
\newenvironment{manualtheorem}[1]{%
  \renewcommand\themanualtheoreminner{#1}%
  \manualtheoreminner
}{\endmanualtheoreminner}

\newtheorem{manualpropositioninner}{Proposition}
\newenvironment{manualproposition}[1]{%
  \renewcommand\themanualpropositioninner{#1}%
  \manualpropositioninner
}{\endmanualpropositioninner}

\newtheorem{manualcorollaryinner}{Corollary}
\newenvironment{manualcorollary}[1]{%
  \renewcommand\themanualcorollaryinner{#1}%
  \manualcorollaryinner
}{\endmanualcorollaryinner}

\theoremstyle{definition}

\theoremstyle{remark}

\usepackage{hyperref}
\usepackage{url}

\title{Discovering Symmetries in Neural Network Parameter Spaces}

\author{%
\normalfont
\begin{tabular}[t]{@{}p{0.48\textwidth}@{\hspace{0.08\textwidth}}p{0.48\textwidth}@{}}
\textbf{Bo Zhao} & \textbf{Nima Dehmamy} \\
Harvard University & IBM Research \\
\texttt{bo\_zhao@seas.harvard.edu} & \texttt{nima.dehmamy@ibm.com} \\[1.4em]
\textbf{Robin Walters} & \textbf{Rose Yu} \\
Northeastern University & University of California, San Diego \\
\texttt{r.walters@northeastern.edu} & \texttt{roseyu@ucsd.edu}
\end{tabular}%
}

\iclrfinalcopy % Uncomment for camera-ready version, but NOT for submission.
\begin{document}

\maketitle

\begin{abstract}
Parameter space symmetries are important for understanding neural networks' loss landscape, training dynamics, and generalization. 
However, systematically identifying these symmetries remains a challenge. 
In this paper, we formalize data-dependent parameter symmetries and characterize loss invariance and the group-action axioms through infinitesimal conditions, which provide objectives for jointly learning group generators and nonlinear action maps. 
Our framework systematically uncovers parameter symmetries, including previously unknown ones. 
To study larger networks, we establish conditions under which subnetwork symmetries extend to the full model. 
The same construction gives an explicit family of finite-batch symmetries, providing both analytical examples and a foundation for discovery through small subnetworks.
Using the infinitesimal characrization and subnetwork construction, we implement a framework for automated discovery of parameter symmetries, and successfully uncovered symmetries in various architectures, including pretrained transformer models.
\end{abstract}

\section{Introduction}
Parameter space symmetry, or loss-invariant transformation of parameters, influences various aspects of deep learning theory.
Continuous symmetry connects groups to their orbits, revealing topological properties such as the dimension \citep{zhao2023symmetries} and connectedness \citep{zhao2023understanding} of the minimum. 
Parameter symmetry also influences training dynamics through the associated conserved quantities of gradient flow \citep{kunin2021neural} and by steering stochastic gradient descent towards certain favored solutions \citep{ziyin2023symmetry}.
Additionally, symmetry provides a tool to perform optimization within a loss level set, with successful applications in accelerating optimization \citep{armenta2023neural, zhao2022symmetry}. 
Other applications of parameter symmetry include model compression \citep{ganev2021universal, sourek2021lossless}, model alignment and merging \citep{zhang2025beyond}, efficient sampling in Bayesian neural networks \citep{wiese2023towards}, and equivariant architectures for weight space learning \citep{navon2023equivariant,zhou2023permutation}.

Despite the wide range of applications, our knowledge of parameter space symmetries remains limited.
In particular, known symmetries often cannot account for all loss-invariant parameter transformations. 
While several frameworks have been developed to unify known symmetries, whether the symmetries in current literature are complete remains an open question \citep{zhao2026symmetry}. 
The lack of a systematic approach necessitates deriving symmetries from scratch for each new architecture, creating barriers for broader application of parameter symmetries.

In this paper, we present an automated approach for discovering parameter space symmetries and their group actions in neural networks. 
To define the search space, we formalize the definition of data-dependent symmetries and derive infinitesimal conditions for loss invariance and the group-action axioms. 
These characterizations connect local parameter changes to finite transformations and provide differentiable objectives for learning nonlinear action maps. 
We also bound finite composition and loss errors in terms of the corresponding differential residuals.

While directly searching for symmetries in modern architectures with billions of parameters is prohibitively expensive, we show that large networks often inherit symmetries from their components or subnetworks. Analyzing these smaller networks reduces the dimension of the discovery problem. 
However, a symmetry of a smaller component need not automatically induce a symmetry of the full network, so this reduction requires understanding when local symmetries can be lifted.
For smooth two-layer blocks, we prove that a class of output-preserving transformations can be decomposed into transformations of small subnetworks.
The proof also constructs a new family of data-dependent symmetries.
These results connect the structure of the network to both the construction and the decomposition of its symmetries.

Our main contributions are:
\begin{itemize}
    \item Exact infinitesimal characterizations of data-dependent parameter symmetries, together with bounds on finite errors when the differential conditions hold approximately.
    \item Conditions for lifting subnetwork symmetries and decomposing a class of symmetries into small-subnetwork symmetries, together with an explicit family of finite-batch actions.
    \item A framework for automated discovery of parameter symmetries building on the infinitesimal characterization and subnetwork construction, with demonstrations of learned nonlinear actions in pretrained transformers and other architectures.
\end{itemize}

\section{Related Work}
\paragraph{Parameter space symmetry.} Parameter symmetries are loss-invariant transformations on neural network parameters, often in the form of group actions. 
Symmetry often arises from equivariance of common activation functions \citep{godfrey2022symmetries}, including invertible linear transformations in linear networks, rescaling in homogeneous networks \citep{badrinarayanan2015symmetry, du2018algorithmic, petzka2020notes}, radial rescaling in radial neural networks \citep{ganev2021universal}, and translation in softmax and scaling in batchnorm functions \citep{kunin2021neural}.
Permutation and sign flips are often the only function-preserving transformations of tanh networks \citep{chen1993geometry}.
ReLU networks, however, possess symmetries beyond the well-known rescaling \citep{grigsby2023hidden}. 
The existence and number of symmetries in most other architectures remain as open questions.

\paragraph{Data-dependent symmetry.} 
While the above symmetries preserve the loss on all data, a relaxed definition, data-dependent symmetry, only requires loss invariance on a subset of data. 
\citet{zhao2023symmetries} gave examples of such symmetries with nontrivial data dependency, though their constructions are limited to minibatches of size one and are difficult to generalize across architectures.
This motivates more systematic discovery methods.
The concept of a symmetry dependent on data has also appeared in adjacent fields. 
For example, \citet{moskalev2023genuine} observe that learned data invariance in neural networks is strongly conditioned on data and breaks under data distribution drift;
\citet{sonoda2023joint} define a joint group action on data and parameters as part of a new proof of universal approximation theory.
% Eli's lecture (10/10/2024): Local data-dependent symmetry, defined by an action from the semigroup G to a smooth function on an epsilon ball around a point $\theta_0$ in the parameter space.

\paragraph{Discovering and measuring symmetry.}
Various work explores learning continuous symmetries by identifying generators of Lie groups \citep{krippendorf2020detecting, moskalev2022liegg, dehmamy2021automatic, yang2023generative, gabel2023learning}, including cases with nonlinear group actions \citep{yang2023latent,shaw2024symmetry,shaw2025continuous,hu2025explicit, ko2024learning}.
We build on this approach to discover data-dependent group action in high-dimensional parameter spaces.
While learning discrete symmetry \citep{zhou2021meta,karjol2024unified} and distributions of symmetry \citep{benton2020learning, romero2022learning, urbano2023self} are also relevant, they are not the primary focus of this paper. 

Extracted symmetry is often evaluated locally, by measuring function changes under infinitesimal symmetry transformations \citep{gruver2022lie} or by comparing tangent spaces of orbits under the learned group and the true symmetry group \citep{portilheiro2023quantifying}.
We use local invariance, similar to that defined in \citep{gruver2022lie, moskalev2022liegg}, as a preservation objective and additionally enforce the infinitesimal action law to learn consistent finite transformations.

\section{Parameter Space Symmetry}
In this section, we formalize the concept of data-dependent parameter symmetries and characterize them through their infinitesimal generators. 
The infinitesimal characterization identifies when a smooth parameter map is a loss-preserving group action, and yields the differentiable objectives used for discovery in Section~\ref{sec:automatic-discovery}.

\subsection{Data-Dependent Group Action and Symmetry}
Let $\Par$ be the space of parameters and $\Data$ be the space of data. 
In this paper, we consider loss functions of the form $L\colon \Par \times \Data \to \R$, which map parameters and a single data point to a real number.
We also use the same notation when the loss is applied to multiple data points at once $L\colon \Par \times \Data^d \to \R^d$ for $d \in \mathbb{N}$.
The same symmetry definition applies to vector- or matrix-valued network outputs, with invariance required componentwise.

Let $G$ be a group. Consider a map $a$, which defines a map for every data batch of size $d \in \mathbb{Z}^+$:
\begin{align}
\label{eq:map-a}
    a\colon \Data^d &\to (G \times \Par \to \Par) \cr 
    X &\mapsto (a_X\colon g, \theta \mapsto \theta').
    % x &\mapsto a_x(g, \theta) = \theta'
\end{align}
The map $a$ is a \emph{generalized group action} on $\Par$ if $a_X$ is a group action for every data batch $X$, meaning that it satisfies the following axioms:
\begin{align*}
    \text{identity:}& \quad a_X(I, \theta) = \theta, \qquad \forall X \in \Data^d, ~~\forall \theta \in \Par. \cr
    \text{associative law:}&  \quad a_{ X}(g_2, a_X(g_1, \theta)) = a_X(g_2 g_1, \theta), 
    \qquad \forall g_1, g_2 \in G, ~~\forall X \in \Data^d, ~~\forall \theta \in \Par. 
    % \text{associative law:}&  \qquad a_{g_1 X}(g_2, a_X(g_1, \theta)) = a_X(g_2 g_1, \theta), \qquad \forall g_1, g_2 \in G, ~~\forall X \subseteq \Data^d, ~~\forall \theta \in \Par. 
    % cocycle condition / compatibility
\end{align*}

We introduce our first definition formalizing the notion of data-dependent symmetry. A group action $a$ is \emph{parameter space symmetry of $L$} if it additionally satisfies
\begin{align*}
    \text{loss invariance:}& \quad L(a_X(g, \theta), X) = L(\theta, X),
    \qquad \forall g \in G, ~~\forall X \in \Data^d, ~~\forall \theta \in \Par.
\end{align*}
A function $L$ has a \emph{$G$-symmetry} if there exists a loss-invariant group action $a$. 
We refer to $G$ as a symmetry group of $L$. 
Additionally, the action $a$ is termed a non-trivial \emph{data-dependent group action} or symmetry if the map \eqref{eq:map-a} has a non-trivial dependency on $X$. 
That is, $a$ is data-dependent if there exists $X_1, X_2 \in \Data^d$, such that $a_{X_1} \neq a_{X_2}$. 
% Equivalently, $a$ is data-dependent if it cannot be written in the form of $G \times \Par \to \Par$.

\subsection{Infinitesimal Symmetry}
\label{sec:infinitesimal-symmetry}

A smooth action determines vector fields describing its instantaneous changes to the parameters. We next characterize loss invariance and composition through differential conditions on these fields and the action map. 
Proofs and examples are given in Appendix~\ref{appendix:infinitesimal-symmetry-examples}.

The following theorem characterizes loss invariance through the derivative of $L$ in the directions generated by the group's infinitesimal action.
We denote the Lie algebra of $G$ by $\mathfrak g=T_IG$.

\begin{theorem}[Infinitesimal characterization of loss invariance]
\label{thm:infinitesimal-symmetry}
Let $G$ be a connected Lie group, let $\Par$ be a smooth manifold, and let $a\colon\Data^d\to(G\times\Par\to\Par)$ be a group action such that each $a_X$ is smooth. 
Let $L\colon\Par\times\Data^d\to\mathbb R^d$ satisfy $L(\cdot,X)\in C^1(\Par,\mathbb R^d)$ for every $X$.
Then $a$ is a parameter space symmetry of $L$ if and only if
\begin{equation}
\label{eq:infinitesimal-symmetry}
    D_\theta L\big|_{\theta,X}
    \bigl(D_ga_X\big|_{I,\theta}(h)\bigr)=0
\end{equation}
for every $\theta\in\Par$, $X\in\Data^d$, and $h\in\mathfrak g$.
\end{theorem}

Here $D_\theta L|_{\theta,X}\colon T_\theta\Par\to\mathbb R^d$ and $D_ga_X|_{I,\theta}\colon\mathfrak g\to T_\theta\Par$.

\begin{proof}[Proof sketch.]
Consider the curve $\gamma(t)=a_X(\exp(th),\theta)$.
Loss invariance implies \eqref{eq:infinitesimal-symmetry} by differentiation at $t=0$. Conversely, the action law gives $\dot\gamma(t)=D_ga_X|_{I,\gamma(t)}(h)$. 
If \eqref{eq:infinitesimal-symmetry} holds at every parameter point, the chain rule shows that $L(\gamma(t),X)$ is constant.
Every element of a connected Lie group is a finite product of exponentials, so the action law extends this invariance to all of $G$.
\end{proof}

Equation \eqref{eq:infinitesimal-symmetry} states that the derivative of each component of $L$ vanishes in the directions in parameter space generated by the infinitesimal action $D_g a_X\big|_{I,\theta}(h)$.
Thus, moving along these directions does not change the loss to first order, reflecting the invariance of $L$ under the group action.

We similarly characterize the group-action axioms infinitesimally. To obtain the condition for composition, we hold one group element fixed and differentiate with respect to an infinitesimal left multiplication of the other.

We use the following notation in the theorem. For $g\in G$, let $R_g(q)=qg$ denote right multiplication.
For a smooth map satisfying $a_X(I,\theta)=\theta$, let $v_{X,h}(\theta):=D_g a_X\big|_{I,\theta}(h)$ denote the infinitesimal action generated by $h\in\mathfrak g$.
As a function of $\theta$, this is a smooth vector field on $\Par$.

\begin{theorem}[Infinitesimal characterization of group actions]
\label{thm:infinitesimal-associative}
\label{thm:infinitesimal-action}
Let $G$ be a connected Lie group, let $\Par$ be a smooth manifold,
and let $a\colon\Data^d\to(G\times\Par\to\Par)$ be a smooth map
satisfying $a_X(I,\theta)=\theta$ for every $X$ and $\theta$.
Then $a_X$ is a group action for every $X$ if and only if
\begin{equation}
\label{eq:infinitesimal-action-transport}
    D_g a_X\big|_{g,\theta}\bigl((dR_g)_Ih\bigr)
    =v_{X,h}\bigl(a_X(g,\theta)\bigr)
\end{equation}
for all $g\in G$, $\theta\in\Par$, $X\in\Data^d$, and
$h\in\mathfrak g$. 
\end{theorem}

Equation~\eqref{eq:infinitesimal-action-transport} shows that the
differential of each orbit map $g\mapsto a_X(g,\theta)$ transports
the right-invariant vector field generated by $h$ on $G$ to the
vector field $v_{X,h}$ on parameter space. For a matrix Lie group, $(dR_g)_Ih=hg$. 
At the identity, associativity implies the Lie-bracket compatibility condition $[v_{X,h_1},v_{X,h_2}]=-v_{X,[h_1,h_2]},$ where the vector-field bracket uses the convention $[u,w]f=u(wf)-w(uf)$. 
Proofs and the additional global conditions are given in Appendix \ref{app:infinitesimal-action}.

\section{Building New Symmetry from Known Ones}
\label{sec:building-symmetries-from-known-ones}

One way to identify symmetries in a large network is by examining its components or subnetworks. Despite often having billions of parameters, neural networks typically consist of a limited set of function families, such as fully connected layers, attention mechanisms, and activation functions. We first establish when symmetries of these components extend to the full network. We then give conditions under which small-subnetwork transformations generate an entire class of output-preserving maps, together with an explicit finite-batch action.

\subsection{Lifting subnetwork symmetries}
\label{sec:lifting-subnetwork-symmetries}

We identify symmetries of a larger network by preserving the outputs of selected subnetworks. 
A naive choice of parameter subset is not sufficient: preserving a computation on a subset of parameters need not preserve the full network if those parameters influence the surrounding computation in other ways.
Proposition~\ref{prop:symmetry-from-subnetworks} formalizes when and how we can obtain symmetries of larger networks from those of their subnetworks.
If a loss function $L$ depends on a subset of the parameters solely through a subnetwork $f$, then any symmetry of $f$ will also preserve $L$:
\begin{proposition}
\label{prop:symmetry-from-subnetworks}
    Let $L\colon \Par \times \Data^d \to \R^d$ where the parameter space $\Par$ is a product space $\Par = \Par_1 \times \Par_2$.
    Suppose for some spaces $S$ and $T$, there exist functions 
    $h\colon \Par_1 \times \Data^d \to S$, 
    $f\colon \Theta_2 \times S \to T$
    and $j\colon (\Theta_1 \times T) \times  \Data^d \to \R^d$,
    such that for every $\theta=(\theta_1, \theta_2) \in \Par$ and $X \in \Data^d$, 
    $L(\theta, X) = j \big( (\theta_1, f \big(\theta_2, h(\theta_1, X) \big)), X \big)$.
    If $a\colon S \to (G \times \Par_2 \to \Par_2)$ is a $G$-symmetry of $f$, then there is an induced $G$-symmetry of $L$, $a'\colon \Data^d \to (G \times \Par \to \Par)$, defined by $a'_X(g, (\theta_1, \theta_2)) =  \big( \theta_1, a_{h(\theta_1, X)}(g, \theta_2) \big)$.
\end{proposition}

The commutative diagram below summarizes the factorization, where $p_1\colon \Theta \to \Theta_1$, $p_2\colon \Theta \to \Theta_2$ are projections,  $\id_1\colon \Theta_1 \to \Theta_1$ and $\id_2\colon \Theta_2 \to \Theta_2$ are identity maps, and $X \in \Data^d$ is a batch of data. 
The spaces $S$ and $T$ can be interpreted as intermediate feature spaces in the neural network. 
In this factorization, $h$ does not depend on $\Theta_2$, and $j$ depends on $\Theta_2$ only through the output of $f$.
This effectively confines $L$'s dependency on $\Theta_2$ to the transformation defined by $f$, ensuring that any transformation on $\Theta_2$ not altering the output of $f$ will not affect the output of $L$.
Consequently, symmetries identified in the smaller network $f$ can be extrapolated to the larger network $L$. 

\begin{center}
\begin{tikzpicture}
    \matrix (m) [matrix of math nodes, row sep=2.5em, column sep=8em, minimum width=2em]{
        \Theta & & \mathbb{R}^d\\
        \Theta_1 \times \Theta_2 \times \Theta_1 & \Theta_1 \times \Theta_2 \times S & \Theta_1 \times T \\
    };
    \path[-stealth]
    (m-1-1) edge node [above] {$L(\cdot, X)$} (m-1-3)
            edge node [left] {$p_1 \times p_2 \times p_1$} (m-2-1)
    (m-2-1) edge node [above] {$\id_1 \times \id_2 \times h(\cdot, X)$} (m-2-2)
    (m-2-2) edge node [above] {$\id_1 \times f(\cdot, \cdot)$} (m-2-3)
    (m-2-3) edge node [right] {$j(\cdot, X)$} (m-1-3);
\end{tikzpicture}
\end{center}

The next two corollaries apply Proposition~\ref{prop:symmetry-from-subnetworks} to show that some symmetries are preserved as networks scale up, either through increasing the dimensionality of a layer or through adding additional layers.
We call $\sigma$ row-wise if it applies a fixed function $\sigma_{row}\colon\mathbb R^k\to\mathbb R^k$ independently to each row of its input. Element-wise functions are a special case.

\begin{corollary}
\label{cor:symmetry-in-wider-networks}    
    Consider the parameter space $\Par(m,h,n)=\R^{m\times h}\times\R^{h\times n}$ and data space $\Data(n,k)=\R^{n\times k}$. Let $\sigma$ apply the same row function $\sigma_{row}\colon\mathbb R^k\to\mathbb R^k$ at every hidden width. Define $L_{mnhk}\colon\Par(m,h,n)\times\Data(n,k)\to\R^{m\times k}$ by $L_{mnhk}((U,V),X)=U\sigma(VX)$. If $L_{mnhk}$ has a $G$-symmetry, then $L_{mnh'k}$ has a $G$-symmetry for every $h'>h$.
\end{corollary}

\begin{corollary}
\label{cor:symmetry-in-deeper-networks}
    Let $\Par = \Par_1 \times ... \times \Par_l$ be a parameter space.
    Consider a list of spaces $V_0 = \Data^d$, $V_l = \R^d$, and $V_1$, ..., $V_{l-1}$.
    Let $L\colon \Par \times \Data^d \to \R^d$ be a function defined recursively by $\{L_i\}_{i=1}^l$ with $L_i\colon \Theta_i \times V_{i-1} \to V_{i}$, such that $L = \phi_l$ where $\phi_i = L_i(\theta_i, \phi_{i-1}) \in V_i$ and $\phi_0 = X$. 
    If for some $1 \leq i \leq l$, $L_i$ has a $G$-symmetry, then $L$ has a $G$-symmetry.
\end{corollary}

Both corollaries follow by factoring and applying Proposition \ref{prop:symmetry-from-subnetworks}.
The corresponding $h$, $f$, and $j$ are deferred to Appendix \ref{appendix:building-symmetries-from-known-ones}.
Figure \ref{fig:symmetry-from-subnetwork} (a-b) shows the subset of parameters ($\Par_2$) whose symmetries lift to the larger network.
% These are the subnetworks where symmetries are assumed to be known and which the larger network inherits.
Figure \ref{fig:symmetry-from-subnetwork} (c) shows why this need not hold for an arbitrary subset of parameters: $\Theta_2$ omits part of the input and output fan for an intermediate neuron, so the full network depends on $\Theta_2$ through information not captured by the subnetwork output.

\begin{figure*}[ht]
\centering
\includegraphics[width=1.0\columnwidth]{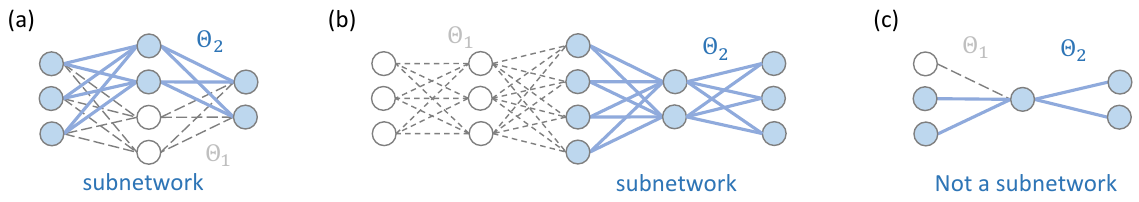}
\caption{Examples of when symmetries in subnetworks are symmetries in the original network (a-b), and when they are not (c).
Blue edges denote the transformed parameters $\Par_2$; gray dashed edges denote the fixed parameters $\Par_1$. 
}
\label{fig:symmetry-from-subnetwork}
\end{figure*}

In addition to obtaining symmetries from those in smaller networks, we can also get symmetries for a loss function over data batches with a certain size, if we know there is a symmetry for this function over larger data batches.
Concretely, an action preserving the pointwise losses on batches of size $d$ induces an action on every smaller positive batch size, by using a fixed padding of the batch.

\begin{proposition}
\label{prop:symmetry-with-smaller-data-batch}
    Let $L_d\colon \Par \times \Data^d \to \R^d$ be a function that is applied pointwise on each of $d$ data points in a data batch.
    If $L_d$ admits a $G$-symmetry, then $L_{d'}$ admits a $G$-symmetry for all $d' < d$. 
\end{proposition}

\subsection{When small subnetworks are sufficient}
\label{sec:small-subnetworks-sufficient}

The subnetwork perspective also yields a completeness result. Under the following conditions, small-subnetwork transformations generate an entire class of symmetry in the original network. 
This allows us to apply divide and conquer to symmetry discovery, making the problem computationally tractable.

Fix $X\in\mathbb R^{n\times k}$ and consider $F_X(U,V)=U\sigma(VX)$ with $U\in\mathbb R^{m\times H}$, $V\in\mathbb R^{H\times n}$, $H>k$, and smooth row-wise $\sigma$. Choose $k$ hidden units $B$, let $C$ be their complement, and assume the domain
\begin{equation}
\label{eq:subnet-regular-domain}
    \Omega_B=\{(U,V):\det[\sigma(VX)]_B\neq0\}
\end{equation}
is nonempty. A hidden-unit block comprises one row of $V$ and the matching column of $U$; subscripts select these blocks.

\begin{theorem}[Decomposition into subnetwork transformations]
\label{thm:subnet-map-generation}
Let $(\phi_t)_{t\in[0,1]}$ be a smooth family of diffeomorphisms of $\Omega_B$ satisfying $\phi_0=\mathrm{id}$ and $F_X\circ\phi_t=F_X$ for every $t\in[0,1]$.
Assume that all $\phi_t$ equal the identity outside the same compact set $K\subset\Omega_B$.
Then $\phi_1 = \psi_N \circ \ldots \circ \psi_1$ for finitely many output-preserving diffeomorphisms $\psi_i$, each changing at most $k+1$ hidden-unit blocks and preserving the combined output of those units on $X$.
\end{theorem}

When $H>k+1$, every factor $\psi_i$ acts on a proper subnetwork. 
The factors may depend on unchanged parameters.

The output coordinates used in the proof additionally give an explicit family of finite-batch actions.
% The output coordinates used in the proof also give an explicit family. 
For $P\in\mathbb R^{m\times(H-k)}$ and $Q\in\mathbb R^{(H-k)\times n}$, consider
\[
    V'_B=V_B,\qquad V'_C=V_C+Q,\qquad U'_C=U_C+P.
\]
The units in $B$ compensate for the output change by setting
\[
    U'_B
    =\bigl[F_X(U,V)-U'_C\sigma(V'_CX)\bigr]
      \bigl[\sigma(V_BX)\bigr]^{-1}.
\]
This is the unique compensating value because $\sigma(V_BX)$ is invertible on $\Omega_B$. Keeping $V_B$ fixed preserves this domain.
Changing a single entry of $P$ or $Q$ affects one unit in $C$ and the compensating output weights of the $k$ units in $B$.

\begin{corollary}[An explicit data-dependent symmetry]
\label{thm:subnet-compensating-action}
\label{cor:subnet-compensating-action}
The maps $a_X^B((P,Q),(U,V))=(U',V')$ defined above form a smooth free action of the additive group $\mathbb R^{(H-k)(m+n)}$ on $\Omega_B$ and preserve $F_X$. This action is generated by commuting one-parameter subgroups, each changing at most $k+1$ hidden-unit blocks and preserving their combined output.
\end{corollary}

% Successive transformations add their $(P,Q)$ increments while keeping $F_X$ and $V_B$ fixed, so the inverse uses $(-P,-Q)$. 

The corollary is of independent interest. 
It defines a closed-form data-dependent action on $\Omega_B$ for finite batches, extending the scope of the single-input constructions in past work without requiring the activation to preserve full rank \citep{zhao2023symmetries}.
The subgraph view reveals the mechanism that enables this family of symmetry: one unit changes, $k$ selected units compensate, and their combined output stays fixed. 
Proposition~\ref{prop:symmetry-from-subnetworks} then lifts each such symmetry to the full network on the corresponding invariant domain. 
Proofs and an illustrative example can be found in Appendix~\ref{app:small-subnetworks-sufficient}.

\section{Automatic Discovery of Parameter Space Symmetry}
\label{sec:automatic-discovery}

Using the infinitesimal characterization derived in Section~\ref{sec:infinitesimal-symmetry}, we construct an automated framework for discovering parameter space symmetries. We describe joint learning of matrix generators and a nonlinear action map, define the loss terms and regularizers, and show how infinitesimal errors control loss invariance and composition errors.

\subsection{Enforcing Loss Invariance and Group Axioms}
\label{sec:discovery-objectives}

Given a function $L$, our goal is to learn a group and its action on $\mathbb R^p$. We learn $r$ generators $h_1,\ldots,h_r\in\mathbb R^{s\times s}$ subject to $\|h_i\|_F=1$. These matrices define the group
\begin{equation}
\label{eq:learned-generated-group}
    G=\left\{\exp(t_mh_{i_m})\cdots\exp(t_1h_{i_1})
    : m\ge0,\ t_j\in\mathbb R,\ i_j\in\{1,\ldots,r\}\right\}.
\end{equation}
% The empty product is $I$. 
Group elements compose by matrix multiplication, and inverses reverse the factors and negate their coefficients. The matrix size $s$ and number of generators $r$ specify the search class.
% ; $r$ need not equal the dimension of the generated Lie algebra.

We parameterize a candidate action using a smooth neural network $b_X$ with a matrix input, enforcing identity exactly,
\begin{equation}
\label{eq:exact-identity-action}
    a_X(g,\theta)=\theta+\bigl(b_X(g,\theta)-b_X(I,\theta)\bigr).
\end{equation}
Thus $a_X(I,\theta)=\theta$ throughout training. Write $v_{X,h}(\theta)=D_ga_X|_{I,\theta}(h)$. We sample parameter--batch pairs $(\theta,X)$, products $g$ with varying lengths, orders, and signed coefficients, and test directions $h$ uniformly from the learned list. If the group $G$ is supplied, we fix $h_i$ to a Frobenius-orthonormal basis of $\mathfrak g$.

We define the following infinitesimal loss terms:
\begin{align}
\label{eq:discovery-invariance-loss}
    \mathcal L_{\mathrm{invariance}}
    &=\mathbb E_{X,\theta,h}\bigl\|D_\theta L\big|_{\theta,X}[v_{X,h}(\theta)]\bigr\|_2^2,\\
\label{eq:discovery-associativity-loss}
    \mathcal L_{\mathrm{assoc}}
    &=\mathbb E_{X,\theta,g,h}\bigl\|D_ga_X\big|_{g,\theta}(hg)-v_{X,h}(a_X(g,\theta))\bigr\|_2^2.
\end{align}
These terms encourage loss invariance and the composition condition in Theorem~\ref{thm:infinitesimal-associative}. Proposition~\ref{prop:learned-generators-suffice} shows that testing the learned generating directions suffices when the residuals vanish everywhere. To control finite transformations directly, we also use
\begin{equation}
\label{eq:discovery-finite-composition-loss}
    \mathcal L_{\mathrm{comp}}
    =\mathbb E_{X,\theta,g_1,g_2}
    \bigl\|a_X(g_2,a_X(g_1,\theta))-a_X(g_2g_1,\theta)\bigr\|_2^2.
\end{equation}
The batch $X$ stays fixed within every composition. We evaluate the residuals at both original and transformed parameter points. Derivatives are computed by automatic differentiation, with sampling details in Appendix~\ref{app:discovery-objectives}.

\subsection{Regularization}
\label{sec:discovery-regularization}

To prevent the learned group action from becoming trivial, we encourage infinitesimal parameter changes with mean squared norm close to $\beta^2$, where $\beta>0$. When learning multiple generators, we also encourage distinct parameter changes \citep{yang2023generative,ko2024learning}. We apply these regularizers to $v_i=v_{X,h_i}(\theta)$,
\begin{align}
\label{eq:generator-scale-regularizer}
    \mathcal L_{\mathrm{scale}}
    &=\frac1r\sum_{i=1}^r\left(\frac{\mathbb E_{X,\theta}\|v_i\|_2^2}{\beta^2}-1\right)^2,\\
\label{eq:generator-diversity-regularizer}
    \mathcal L_{\mathrm{div}}
    &=\frac{2}{r(r-1)\beta^4}\sum_{1\le i<j\le r}\bigl(\mathbb E_{X,\theta}\langle v_i,v_j\rangle\bigr)^2.
\end{align}
We set $\mathcal L_{\mathrm{div}}=0$ when $r=1$. These terms control the average size of the induced fields and penalize correlations between them. The expectations use the original parameter and batch samples.

The base training objective, before fixed normalization, is
\begin{equation}
\label{eq:objective-weighted}
\begin{aligned}
    \min_{b,h_1,\ldots,h_r}\quad
    &\gamma_1\mathcal L_{\mathrm{invariance}}
     +\gamma_2\mathcal L_{\mathrm{comp}}
     +\gamma_3\mathcal L_{\mathrm{assoc}}
    +\gamma_4\mathcal L_{\mathrm{scale}}
     +\gamma_5\mathcal L_{\mathrm{div}},
\end{aligned}
\end{equation}
where $\gamma_1,\gamma_3,\gamma_4,\gamma_5>0$ and $\gamma_2\ge0$.

We call this the hybrid objective when $\gamma_2>0$, since it combines the infinitesimal conditions with a direct test of finite composition. 
Setting $\gamma_2=0$ gives the infinitesimal-only objective. 
To assess the value of the infinitesimal conditions, we also train a finite-constraint baseline that checks output preservation and composition directly, with the same regularizers. 
Appendix~\ref{app:discovery-objectives} gives the exact variants and training normalization.

\subsection{Approximate Symmetry}
\label{sec:approximate-symmetry}

We prove that small identity and infinitesimal errors give small loss-preservation and composition errors over bounded group paths.
These bounds support using the infinitesimal characterizations as discovery objectives.

On a data batch $X$, let $a_X\colon G\times\mathbb R^p\to\mathbb R^p$ be a smooth map on a connected Lie subgroup $G\subseteq\mathrm{GL}(s,\mathbb R)$, with $L(\cdot,X) \in C^1$. 
% Use Euclidean norms on the parameter and loss spaces and the Frobenius norm on $\mathfrak g$. 
Suppose $\|a_X(I,\theta)-\theta\|\le\delta$ and the unsquared residuals in $\mathcal L_{\mathrm{invariance}}$ and $\mathcal L_{\mathrm{assoc}}$ are bounded by $\varepsilon_{\mathrm{loss}},\varepsilon_{\mathrm{act}}\ge0$, uniformly over $\theta\in\mathbb R^p$, $g\in G$, and unit $h\in\mathfrak g$. Equation~\eqref{eq:exact-identity-action} gives $\delta=0$. Assume also a bounded loss derivative and Lipschitz generators, with $M,K\ge0$ such that, for all $\theta,\vartheta\in\mathbb R^p$ and $h\in\mathfrak g$,
$$\|D_\theta L|_{\theta,X}\|_{\mathrm{op}}\le M,\quad \text{and}\quad
    \|v_{X,h}(\theta)-v_{X,h}(\vartheta)\|
    \le K\|h\|\,\|\theta-\vartheta\|.$$
For a piecewise $C^1$ path $\gamma\colon[0,1]\to G$ with finitely many pieces, define its length by
\[
    \ell=\int_0^1\|\dot\gamma(t)\gamma(t)^{-1}\|\,dt,
    \qquad F_K(\ell)=\int_0^\ell e^{Kt}\,dt.
\]
The length measures the accumulated group motion. For example, $\gamma(t)=\exp(th)$ has length $\|h\|$.

\begin{proposition}[Finite symmetry errors]
\label{prop:uniform-finite-defects}
Under the preceding assumptions, for every such path $\gamma$ from $I$ to $g_2$ and every $g_1\in G$, $\theta\in\mathbb R^p$,
\begin{align}
\label{eq:uniform-associativity-error}
    \|a_X(g_2,a_X(g_1,\theta))-a_X(g_2g_1,\theta)\|
    &\leq\delta e^{K\ell}+2\varepsilon_{\mathrm{act}}F_K(\ell),\\
\label{eq:uniform-loss-error}
    \|L(a_X(g_2,\theta),X)-L(\theta,X)\|
    &\leq M\delta+
    (\varepsilon_{\mathrm{loss}}+M\varepsilon_{\mathrm{act}})\ell.
\end{align}
\end{proposition}

For the learned group, the same inequalities hold for a product $g_2$ from \eqref{eq:learned-generated-group} with $\ell=\sum_{j=1}^m|t_j|$, using the corresponding residual and Lipschitz bounds on the normalized generators $h_i$. 
Appendix~\ref{app:associativity-stability} gives the proofs and versions using bounds along individual paths.

\section{Experiments}
\label{sec:experiments}

We use this framework to find parameter symmetries in several common architectures. We first examine known data-independent symmetries and recover the structure of higher-dimensional symmetry groups. We then study nonlinear data-dependent actions and transformations on non-contiguous layers. Finally, we use subnetworks to study larger models, including pretrained Pythia transformers.

The direct-map experiments jointly learn one generator and a nonlinear action, preserving complete network or subnetwork outputs $F_X(\theta)$. Independent evaluations retain all five small-network seeds and all three Pythia seeds. Numerical summaries are medians across seeds of within-seed $95$th-percentile errors; motion uses medians. Appendix~\ref{app:experimental-protocol} describes the sampling and additional results.

\subsection{Learned data-independent symmetries}
\label{sec:experimental-recovery}

We first validate that our method can learn generators for known data-independent symmetries in neural networks. 
We consider a two-layer linear network $F(U,V)=UV$, with $U,V\in\mathbb R^{2\times2}$, and a bias-free ReLU network with widths $(2,4,2)$. 
Their reference symmetries are hidden-basis changes and positive unit rescalings. 
The learned maps approximately preserve outputs and compose consistently while moving the parameters. Their fields follow the reference directions, indicating successful recovery of one-parameter actions (Appendix~\ref{app:experimental-structure}).

An affine-field solver, following \citet{moskalev2022liegg}, additionally recovers symmetry algebras of up to eight dimensions without receiving the reference generators or their dimension. It recovers noncommuting transformations in linear networks, for which reversing the order of two transformations can change the final parameters. The ReLU rescalings commute. This separate specialization and exact verification of its reconstructed generators are described in Appendix~\ref{app:frontier-recovery}.

\subsection{Learned data-dependent symmetries}
\label{sec:experimental-compensation-main}

We next study data-dependent transformations in a two-layer sigmoid network, $F_X(U,V)=U\sigma(VX)$. We fix the incoming weights of $k$ compensating units and the outgoing weights of one moving unit. The action changes the moving unit's incoming weights and the compensators' output weights. 
The analytic reference (Corollary~\ref{cor:subnet-compensating-action}) is withheld from training.

\begin{figure}[t]%[!htbp]
\centering
\begin{minipage}{0.32\linewidth}\centering
\includegraphics[width=\linewidth]{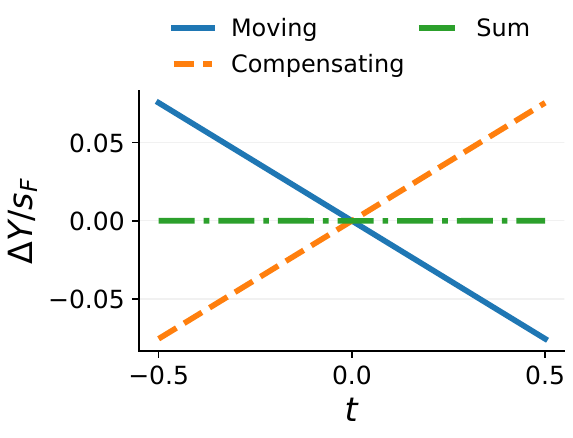}
\small (a) Nonlinear compensation
\end{minipage}\hfill
\begin{minipage}{0.32\linewidth}\centering
\includegraphics[width=\linewidth]{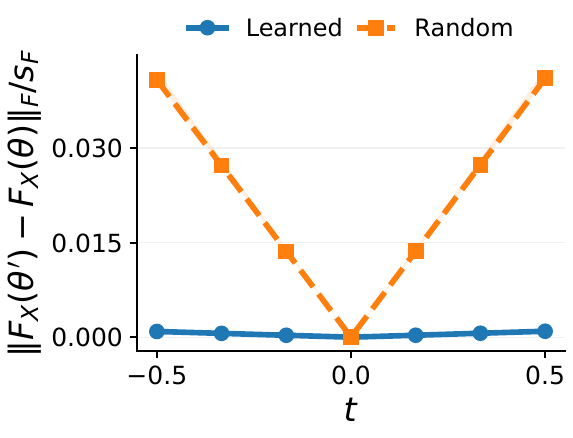}
\small (b) Output along tanh curves
\end{minipage}\hfill
\begin{minipage}{0.32\linewidth}\centering
\includegraphics[width=\linewidth]{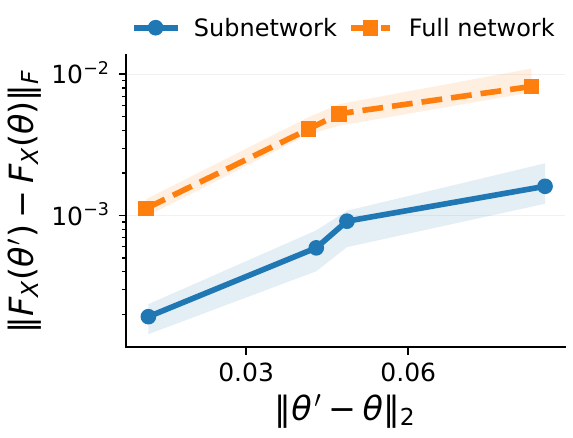}
\small (c) Subnetwork discovery
\end{minipage}
\caption{
Learned compensation, finite preservation, and subnetwork discovery. 
% In (a,b), $\theta'=a_X(\exp(th),\theta)$ with signed coefficient $t$. 
(a) Opposing changes $\Delta Y_j=Y'_j-Y_j$ in unit contributions $Y_j=u_j\sigma(v_jX)$ nearly cancel. $s_F$ is the fixed root-mean-square output norm. (b) Tanh output drift is smaller compared to random translations. (c) Subnetwork discovery gives smaller full-network output error at similar displacement. 
% Panel (b) aggregates median drift and (c) $95$th-percentile error across five seeds; shading spans their seed summaries.
}
\label{fig:experimental-compensation}
\label{fig:experimental-orbits}
\label{fig:experimental-paired-lifting}
\end{figure}

Our method learns to preserve outputs through nonlinear compensation. In the scalar, single-input case, the transformed parameter block changes by a median of $12\%$. The two units make opposing output changes (Figure~\ref{fig:experimental-compensation}(a)), with a very small cancellation error.
% with a cancellation error of $0.060\%$ relative to the sum of their magnitudes. 
The same transformed weights produce much larger changes on fresh inputs, showing that preservation is specific to the conditioning data. Appendix~\ref{app:experimental-comparisons} gives the comparison and two-input results.

The method also learns coordinated transformations on non-contiguous layers. In a three-layer tanh network with widths $(4,16,16,2)$ trained on synthetic regression, we change the first and last weight matrices while fixing the middle layer. On $24$ fixed inputs, output drift along generated curves $\theta(t)=a_X(\exp(th),\theta_0)$ is much smaller than for random translations calibrated to similar parameter motion (Figure~\ref{fig:experimental-orbits}(b)). Here $\theta_0$, $X$, and $h$ stay fixed along each curve.

The infinitesimal conditions improve the consistency of finite transformations. We divide the discrepancy between successive transformations and their group product by the summed displacements of the two steps. In scalar sigmoid, the hybrid objective reduces this composition error from $0.306\%$ under finite constraints to $0.037\%$, at similar motion. Both objectives share the architecture, regularizers, and final step budget, with separate validation-based tuning.

\subsection{Symmetry in subnetworks and pretrained transformers}
\label{sec:experimental-lifting}
\label{sec:experimental-pythia}

Subnetwork discovery preserves full-network outputs more accurately than discovery over all parameters in our paired comparison. We learn an action on $10$ weights incident to one hidden unit in a ReLU network with widths $(2,8,8,1)$, preserving its full boundary output as required by Proposition~\ref{prop:symmetry-from-subnetworks}. At similar parameter displacement, it gives $5.8$ times smaller full-network output error than discovery over all $88$ parameters (Figure~\ref{fig:experimental-paired-lifting}(c)). Both maps use identical evaluation parameters and inputs. This compares the smaller support and boundary-output target together.

We also learn symmetries in pretrained Pythia-160M and Pythia-1B feedforward blocks. For a fixed $k$-token prefix, the action changes one unit's incoming weights and bias and $k$ compensators' outgoing columns, keeping all other parameters fixed. We train on their combined boundary output, then install the edits in float32 and evaluate the complete model at the original checkpoint. 
% The analytic compensation formula is used only as a reference.

\begin{figure}[t]%[!htbp]
\centering
\begin{minipage}{0.485\linewidth}\centering
\includegraphics[width=\linewidth]{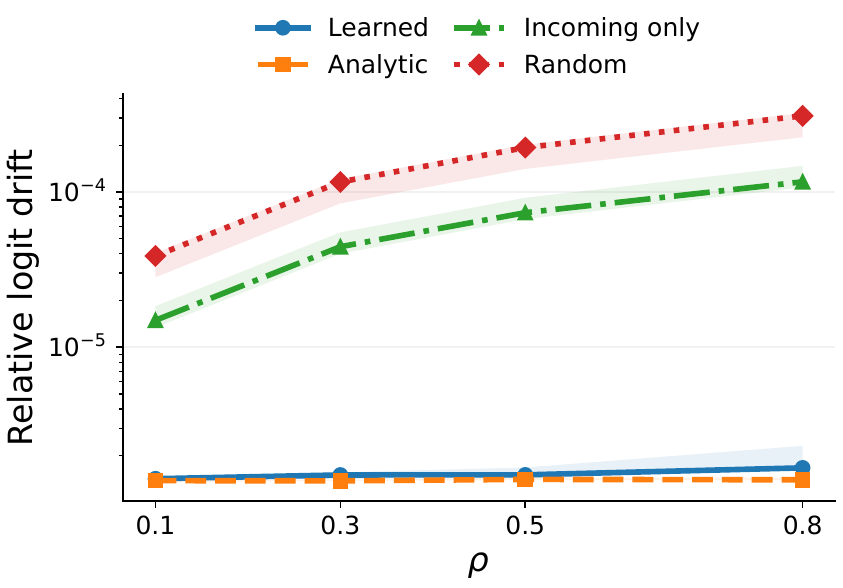}
\small (a) Preserving protected outputs
\end{minipage}\hfill
\begin{minipage}{0.485\linewidth}\centering
\includegraphics[width=\linewidth]{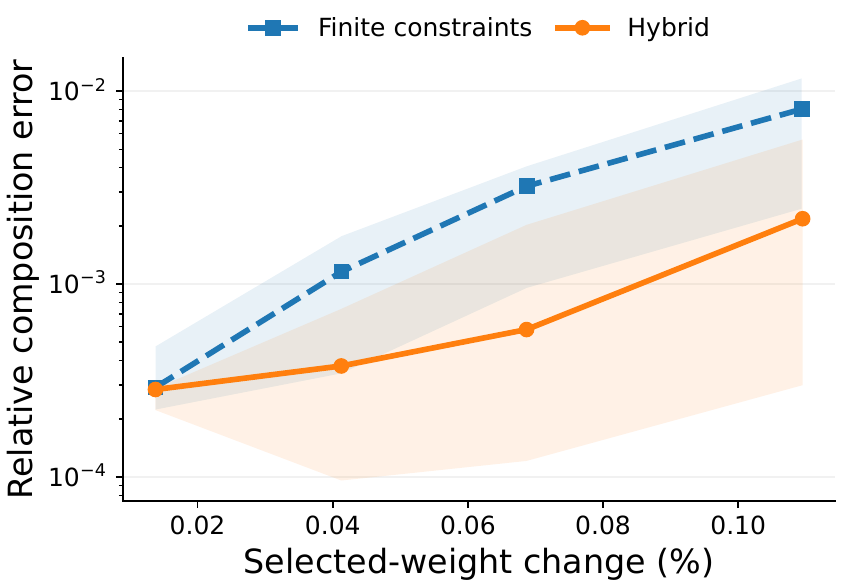}
\small (b) Composing learned transformations
\end{minipage}
\caption{Learned symmetry in Pythia-1B, layer $3$, at the original checkpoint. (a) Relative drift $\|Z'-Z\|_F/\|Z\|_F$ of original and changed logits on one fixed $16$-token prefix. $\rho$ bounds $|t|$ in $\exp(th)$. Learned uses the hybrid objective; analytic solves for compensation; incoming-only omits it. Analytic and random references are motion-calibrated. (b) Composition error versus $100\|\theta'_S-\theta_S\|_2/\|\theta_S\|_2$ on edited weights $S$. }
\label{fig:pythia-controls}
\end{figure}

Learned compensation sharply reduces protected-output drift in Pythia-1B (Figure~\ref{fig:pythia-controls}(a)). At coefficient radius $\rho=0.5$, the selected weights change by a median of $0.069\%$. Protected-logit drift is approximately $129$ times smaller than for a random translation at similar motion. Keeping the learned incoming update while removing compensation increases drift by $49$ times, despite reducing the total displacement.

Infinitesimal training also improves composition after the weights are installed. At nearly identical motion, the hybrid objective reduces stored-weight composition error from $0.321\%$ to $0.058\%$ of the two-step displacement (Figure~\ref{fig:pythia-controls}(b)). 
Each update is rounded to float32 before the next action. 
Pythia-160M shows the same benefits of learned compensation and improved composition (Appendix~\ref{app:experimental-pythia-details}).

\section{Discussion}
In this work, we introduce a framework for formalizing and discovering parameter space symmetries in neural networks. Infinitesimal characterizations and finite-error bounds provide objectives for learning nonlinear actions. Subnetwork lifting and decomposition connect this differential structure to the modular structure of neural networks.
Mathematically, the existence and number of symmetries in neural network parameter spaces remain open questions. 
Whether the number of symmetries is affected by existence of symmetry in data or changes during training are also interesting directions.
The discovered symmetries could be useful in downstream applications such as weight space learning, where architectures respecting known parameter symmetries already process neural network weights effectively \citep{zhang2023neural,lim2024graph,kalogeropoulos2024scale,tran2024monomial}. 
Leveraging the learned symmetries via equivariant architectures \citep{finzi2021practical} or frame averaging \citep{puny2021frame} can potentially improve current approaches.% that are limited to known symmetries.

% \clearpage
\subsection*{AI use statement}

In this work, we used generative AI tools for checking proof correctness, implementing experiment code, and reviewing drafts of the manuscript. 
We have reviewed all AI-assisted work. We take responsibility for the final content of this work.

% \subsection*{Reproducibility statement}

% (This section is \textbf{recommended} and does not count toward the page limit.)

% It is important that the work published in ICLR is reproducible. Authors are
% strongly encouraged to include a paragraph-long Reproducibility Statement at the
% end of the main text (before references) to discuss the efforts that have been
% made to ensure reproducibility. This paragraph should not itself describe
% details needed for reproducing the results, but rather reference the parts of
% the main paper, appendix, and supplemental materials that will help with
% reproducibility. For example, for novel models or algorithms, a link to an
% anonymous downloadable source code can be submitted as supplementary materials;
% for theoretical results, clear explanations of any assumptions and a complete
% proof of the claims can be included in the appendix; for any datasets used in
% the experiments, a complete description of the data processing steps can be
% provided in the supplementary materials. Each of the above are examples of
% things that can be referenced in the reproducibility statement.

% This work was supported in part by the U.S. Army Research Office under Army-ECASE award W911NF-07-R-0003-03, the U.S. Department Of Energy, Office of Science, IARPA HAYSTAC Program, and NSF Grants \#2205093, \#2146343, \#2134274, \#2442658, \#2134178, CDC-RFA-FT-23-0069, DARPA AIE FoundSci and DARPA YFA. 

\bibliography{iclr2027_conference}
\bibliographystyle{iclr2027_conference}

\clearpage
\appendix
\section*{Appendix}

\section{Infinitesimal Characterizations of Parameter Symmetry}
\label{appendix:infinitesimal-symmetry-examples}

This section proves the infinitesimal characterizations in Section~\ref{sec:infinitesimal-symmetry}, recalls the integration result used there, and proves the finite-error guarantees following the discovery objectives in Section~\ref{sec:approximate-symmetry}. Examples are given in Appendix~\ref{app:infinitesimal-examples}.

\subsection{Infinitesimal Formulation for Loss Invariance}
\label{app:infinitesimal-loss-invariance}

\begin{manualtheorem}{\ref{thm:infinitesimal-symmetry}}
Let $G$ be a connected Lie group, let $\Par$ be a smooth manifold, and let $a\colon\Data^d\to(G\times\Par\to\Par)$ be a group action such that each $a_X$ is smooth. Let $L\colon\Par\times\Data^d\to\mathbb R^d$ satisfy $L(\cdot,X)\in C^1(\Par,\mathbb R^d)$ for every $X$. Then $a$ is a parameter space symmetry of $L$ if and only if
\[
    D_\theta L\big|_{\theta,X}
    \bigl(D_ga_X\big|_{I,\theta}(h)\bigr)=0
\]
for every $\theta\in\Par$, $X\in\Data^d$, and $h\in\mathfrak g$.
\end{manualtheorem}

\begin{proof}
Fix $X$ and $h\in\mathfrak g$, and consider the smooth curve $\gamma(t)=a_X(\exp(th),\theta)$. If $a$ is a symmetry of $L$, then
\[
    L(\gamma(t),X)=L(\theta,X)
    \qquad\text{for all }t\in\mathbb R.
\]
Differentiating at $t=0$, using $\gamma(0)=\theta$ and $\left.\frac{d}{dt}\right|_0\exp(th)=h$, gives
\[
    0=D_\theta L\big|_{\theta,X}
      \left[\left.\frac{d\gamma(t)}{dt}\right|_0\right]
     =D_\theta L\big|_{\theta,X}
      \bigl[D_ga_X\big|_{I,\theta}(h)\bigr].
\]

Conversely, suppose \eqref{eq:infinitesimal-symmetry} holds at every parameter point. The action law gives
\[
    \gamma(t+s)=a_X(\exp(sh),\gamma(t)).
\]
Differentiating with respect to $s$ at $0$ yields $\dot\gamma(t)=D_ga_X|_{I,\gamma(t)}(h)$. The chain rule therefore gives
\[
    \frac{d}{dt}L(\gamma(t),X)
    =D_\theta L\big|_{\gamma(t),X}
      \bigl[D_ga_X\big|_{I,\gamma(t)}(h)\bigr]=0.
\]
Consequently every exponential acts loss-invariantly. The subgroup generated by the image of the exponential map contains an identity neighborhood and is therefore open. Its other cosets are open as well, so the subgroup is also closed. Since $G$ is connected, this subgroup is all of $G$. Every $g\in G$ is thus a finite product of exponentials. Repeated use of the action law gives $L(a_X(g,\theta),X)=L(\theta,X)$ for every $g$. The argument holds for every $X$, proving the converse.
\end{proof}

\subsection{Infinitesimal formulation for the group-action axioms}
\label{app:infinitesimal-action}

Throughout this subsection, Lie groups and manifolds are finite-dimensional, and manifolds are without boundary. Smoothness of $a$ means smoothness of $a_X$ in $(g,\theta)$ for each fixed data batch $X$.

\begin{manualtheorem}{\ref{thm:infinitesimal-associative}}
Let $G$ be a connected Lie group, let $\Par$ be a smooth manifold, and let $a\colon\Data^d\to(G\times\Par\to\Par)$ be a smooth map satisfying $a_X(I,\theta)=\theta$ for every $X$ and $\theta$. Define $v_{X,h}(\theta)=D_g a_X|_{I,\theta}(h)$. Then $a_X$ is a group action for every $X$ if and only if
\[
    D_g a_X\big|_{g,\theta}\bigl((dR_g)_Ih\bigr)
    =v_{X,h}\bigl(a_X(g,\theta)\bigr)
\]
for all $g,\theta,X,h$.
\end{manualtheorem}

\begin{proof}
Fix $X$ and suppress it from the notation. First suppose that $a$ is a group action. Associativity gives
\[
    a(\exp(th)g,\theta)
    =a(\exp(th),a(g,\theta)).
\]
Differentiating at $t=0$, and using $\left.\frac{d}{dt}\right|_{0}\exp(th)g=(dR_g)_Ih$, gives
\[
    D_g a\big|_{g,\theta}\bigl((dR_g)_Ih\bigr)
    =D_g a\big|_{I,a(g,\theta)}(h)
    =v_h(a(g,\theta)).
\]

Conversely, suppose that identity and the differential condition hold. For fixed $h$, $g$, and $\theta$, define
\[
    u(t)=a(\exp(th)g,\theta),\qquad
    w(t)=a(\exp(th),a(g,\theta)).
\]
Both curves are defined for every $t\in\mathbb R$. Since
\[
    \frac{d}{dt}\exp(th)g=(dR_{\exp(th)g})_Ih,
\]
the differential condition implies
\[
    \dot u(t)=v_h(u(t)),\qquad \dot w(t)=v_h(w(t)).
\]
The identity axiom gives $u(0)=w(0)=a(g,\theta)$. Uniqueness of integral curves of a smooth vector field therefore yields
\begin{equation}
\label{eq:action-one-parameter-factor}
    a(\exp(th)g,\theta)
    =a(\exp(th),a(g,\theta))
    \qquad\text{for all }t\in\mathbb R.
\end{equation}
No completeness assumption is needed in this step: the two integral curves are already globally defined by the given map $a$.

As shown in Appendix~\ref{app:infinitesimal-loss-invariance}, every element of a connected Lie group is a finite product of exponentials.
Write $g_2=\exp(h_m)\cdots\exp(h_1)$. Repeated application of \eqref{eq:action-one-parameter-factor} shows that both $a(g_2g_1,\theta)$ and $a(g_2,a(g_1,\theta))$ equal the successive application of $a(\exp(h_1),\cdot),\ldots,a(\exp(h_m),\cdot)$ to $a(g_1,\theta)$. This proves associativity.

Associativity and identity imply
\[
    a(g^{-1},a(g,\theta))=\theta,
    \qquad a(g,a(g^{-1},\theta))=\theta.
\]
Thus $a(g,\cdot)$ and $a(g^{-1},\cdot)$ are mutually inverse smooth maps. Finally, both sides of the differential condition are linear in $h$, so checking a basis suffices.
\end{proof}

Every $v_{X,h}$ is complete, with flow $t\mapsto a_X(\exp(th),\theta)$. Here completeness follows from the globally defined map.

\paragraph{Testing generating directions.} The group in Section~\ref{sec:discovery-objectives} is determined by finite products of learned matrix exponentials. The following consequence of the preceding proof shows why testing their generating directions is sufficient.

Let $G$ be a Lie group generated by $\{\exp(th_i):t\in\mathbb R,\ 1\leq i\leq r\}$, with $h_i\in\mathfrak g$. For each fixed $X$, let $a_X\colon G\times\Par\to\Par$ be smooth with $a_X(I,\theta)=\theta$, set $v_{X,h}(\theta)=D_ga_X|_{I,\theta}(h)$, and let $L(\cdot,X)$ be $C^1$.

\begin{proposition}[Symmetry from generating directions]
\label{prop:learned-generators-suffice}
In this setup, $a_X$ is a loss-preserving group action for every $X$ if and only if, for every $g,\theta,X$ and $i\in\{1,\ldots,r\}$,
\begin{equation}
\label{eq:learned-generator-transport}
\begin{aligned}
    D_ga_X\big|_{g,\theta}\bigl((dR_g)_Ih_i\bigr)
    &=v_{X,h_i}(a_X(g,\theta)),\\
    D_\theta L\big|_{\theta,X}[v_{X,h_i}(\theta)]&=0.
\end{aligned}
\end{equation}
\end{proposition}

The selected directions may generate additional Lie algebra directions through commutators. The proposition justifies testing the selected list while the group argument ranges over all its finite products.

\begin{proof}
Fix $X$ and suppress it from the notation. For a loss-preserving action, differentiating $a(\exp(th_i)g,\theta)=a(\exp(th_i),a(g,\theta))$ and $L(a(\exp(th_i),\theta),X)=L(\theta,X)$ at $t=0$ gives the two conditions.

Conversely, suppose they hold. For fixed $i,g,\theta$, define
\[
    u(t)=a(\exp(th_i)g,\theta),\qquad
    w(t)=a(\exp(th_i),a(g,\theta)).
\]
The first condition gives $\dot u=v_{h_i}(u)$ and $\dot w=v_{h_i}(w)$ for all $t\in\mathbb R$. Identity gives $u(0)=w(0)=a(g,\theta)$, so uniqueness of integral curves implies
\[
    a(\exp(th_i)g,\theta)=a(\exp(th_i),a(g,\theta)).
\]
For $g_2=\exp(t_mh_{i_m})\cdots\exp(t_1h_{i_1})$, repeated application of this equality expresses both $a(g_2g_1,\theta)$ and $a(g_2,a(g_1,\theta))$ as the same successive transformations of $a(g_1,\theta)$. Thus $a$ satisfies associativity. Together with identity, this gives a group action whose inverse maps are $a(g^{-1},\cdot)$.

The second condition and the chain rule give
\[
    \frac{d}{dt}L(a(\exp(th_i),\theta),X)
    =D_\theta L\big|_{a(\exp(th_i),\theta),X}
      [v_{h_i}(a(\exp(th_i),\theta))]=0.
\]
Each generating one-parameter action preserves $L$, and finite composition proves invariance on $G$.
\end{proof}

\paragraph{Lie-bracket compatibility.} For smooth vector fields we use the convention $[u,w]f=u(wf)-w(uf)$. For a smooth left action, the generators satisfy
\begin{align}
\label{eq:generator-bracket-compatibility}
    [v_{X,h_1},v_{X,h_2}]&=-v_{X,[h_1,h_2]},\\
\label{eq:generator-bracket-coordinates}
    D_\theta v_{X,h_1}[v_{X,h_2}]
    -D_\theta v_{X,h_2}[v_{X,h_1}]&=v_{X,[h_1,h_2]},
\end{align}
where the second identity is in local coordinates. To see this, fix $X$ and suppress it from the notation. Differentiate $a(g,a(\exp(th),\theta))=a(\exp(t\,\mathrm{Ad}_g h),a(g,\theta))$ at $t=0$ to obtain
\begin{equation}
\label{eq:generator-adjoint-equivariance}
    D_\theta a\big|_{g,\theta}[v_h(\theta)]
    =v_{\mathrm{Ad}_g h}(a(g,\theta)).
\end{equation}
Next take $g=\exp(th_1)$ and $h=h_2$ and differentiate at $t=0$. Using $\left.\frac{d}{dt}\right|_0\mathrm{Ad}_{\exp(th_1)}h_2=[h_1,h_2]$ gives $D_\theta v_{h_1}[v_{h_2}]=v_{[h_1,h_2]}+D_\theta v_{h_2}[v_{h_1}]$, which proves both identities. These relations describe how the learned matrix algebra acts on parameter space.

\paragraph{Constructing an action from infinitesimal generators.} For a simply connected Lie group $G$, the Lie--Palais theorem integrates a linear family $h\mapsto v_{X,h}$ of smooth vector fields into a unique smooth left action if every $v_{X,h}$ is complete and \eqref{eq:generator-bracket-compatibility} holds \citep{palais1957global,tuynman2013integrating}. Completeness means that each integral curve exists for every real time. By Theorem~\ref{thm:infinitesimal-symmetry}, the resulting action preserves a $C^1$ loss exactly when $D_\theta L|_{\theta,X}[v_{X,h}(\theta)]=0$ for every $\theta,X,h$.

\paragraph{An integrated action need not equal a prescribed map.} For $G=(\mathbb R,+)$ and $\Par=\mathbb R$, the map $a(t,\theta)=\theta+t+t^2$ satisfies identity and has complete, compatible generators $v_h(\theta)=h$, but $a(s+t,\theta)-a(s,a(t,\theta))=2st$. Integrating its generators instead gives $A(t,\theta)=\theta+t$. Thus a directly parameterized map still requires \eqref{eq:infinitesimal-action-transport}.

\subsection{Examples}
\label{sec:infinitesimal-symmetry-examples}
\label{app:infinitesimal-examples}

Below are examples of familiar parameter space symmetry and their infinitesimal form. 

\subsubsection{Linear action of matrix groups}
For a linear action, the infinitesimal parameter change is a matrix--vector product. Let $G\subseteq\GL(n)$ be a connected matrix Lie group acting on $\Par=\mathbb R^n$ by $a_x(g,\theta)=g\theta$. Then $v_{x,h}(\theta)=h\theta$, and Theorem~\ref{thm:infinitesimal-symmetry} characterizes invariance of a $C^1$ scalar loss by
\begin{equation}
\label{eq:infinitesimal-symmetry-coordinates}
    D_\theta L\big|_{\theta,x}[h\theta]
    =\sum_{i,j=1}^n
      \frac{\partial L}{\partial\theta_i}(\theta,x)\,h_{ij}\theta_j
    =0
\end{equation}
for every $\theta,x,h$. For vector-valued losses, the condition holds componentwise. 

This has the same algebraic expression as the infinitesimal criterion in Theorem 3.1 of \citet{moskalev2022liegg}, with transformations acting on parameters rather than data. Their theorem characterizes preservation of a regular zero set, whereas ours characterizes preservation of the value of $L$ throughout parameter space.

\subsubsection{Homogeneous two-layer neural network}
Rescaling one weight can be compensated by an inverse rescaling of the next. Consider $L((w_1,w_2),x)=w_2\sigma(w_1x)$, where $\sigma(\alpha z)=\alpha^c\sigma(z)$ for every $\alpha>0$ and $z\in\mathbb R$, with $c>0$. The group $(\mathbb R_{>0},\times)$ acts by
\[
    a_x(\alpha,(w_1,w_2))
    =(\alpha w_1,\alpha^{-c}w_2).
\]
Multiplication of the scaling factors gives the action law, and homogeneity gives $L(a_x(\alpha,(w_1,w_2)),x)=\alpha^{-c}w_2\alpha^c\sigma(w_1x)=L((w_1,w_2),x)$. Differentiating at $\alpha=e^t$, $t=0$, gives
\[
    v_{x,1}(w_1,w_2)=(w_1,-cw_2),
    \qquad
    w_1\frac{\partial L}{\partial w_1}
    -cw_2\frac{\partial L}{\partial w_2}=0.
\]
Finite invariance uses only homogeneity. The derivative identity holds wherever $L(\cdot,x)$ is differentiable, and the global characterization applies when its stated $C^1$ assumption holds.

\subsection{Stability of finite transformations}
\label{app:associativity-stability}
\label{app:finite-symmetry-errors}

We prove and refine the finite-error guarantees in Section~\ref{sec:approximate-symmetry}. The candidate map may have nonzero identity and composition errors. The following estimates track the differential residuals along the paths used to compare finite transformations.

Fix $X$, let $G$ be a Lie group, and let $a_X\colon G\times\mathbb R^p\to\mathbb R^p$ be smooth. Using the canonical identification of Euclidean tangent spaces, set $v_{X,h}(\theta)=D_ga_X|_{I,\theta}(h)$. Let $L(\cdot,X)$ be $C^1$. The differential errors penalized in Section~\ref{sec:discovery-objectives} are
\begin{align}
\label{eq:action-transport-residual}
    r_X(g,\theta;h)
    &=D_ga_X\big|_{g,\theta}\bigl((dR_g)_Ih\bigr)-v_{X,h}(a_X(g,\theta)),\\
\label{eq:loss-infinitesimal-residual}
    s_X(\theta;h)
    &=D_\theta L\big|_{\theta,X}[v_{X,h}(\theta)].
\end{align}
For matrix groups, $(dR_g)_Ih=hg$. Euclidean norms are used after flattening matrix-valued outputs.

Let $\gamma\colon[0,1]\to G$ be a piecewise $C^1$ path with finitely many pieces, from $I$ to $g_2$. Define $h(t)$ by $\dot\gamma(t)=(dR_{\gamma(t)})_Ih(t)$ wherever the derivative exists. For fixed $g_1\in G$ and $\theta\in\mathbb R^p$, write
\[
    u(t)=a_X(\gamma(t)g_1,\theta),\qquad
    w(t)=a_X(\gamma(t),a_X(g_1,\theta)).
\]
The next proposition bounds the difference between applying $g_1$ and $g_2$ successively and applying their product. It only requires control along these two trajectories.

\begin{proposition}[Transport residual controls associativity error]
\label{prop:pathwise-associativity}
In the preceding setup, suppose $\lambda\colon[0,1]\to[0,\infty)$ is integrable and
\[
    \|v_{X,h(t)}(u(t))-v_{X,h(t)}(w(t))\|
    \leq\lambda(t)\|u(t)-w(t)\|
\]
almost everywhere. Set $\Lambda(t)=\int_0^t\lambda(s)\,ds$. Then
\begin{align*}
    \|u(1)-w(1)\|
    &\leq e^{\Lambda(1)}\|u(0)-w(0)\|\\
    &\quad+\int_0^1 e^{\Lambda(1)-\Lambda(s)}
       \Bigl(\|r_X(\gamma(s)g_1,\theta;h(s))\|\\
    &\hspace{39mm}
       +\|r_X(\gamma(s),a_X(g_1,\theta);h(s))\|\Bigr)\,ds.
\end{align*}
\end{proposition}

The left-hand side is the finite composition error. The initial discrepancy is the identity error at $a_X(g_1,\theta)$, and the integral accumulates the transport residuals. Thus, small identity error and small transport residuals along both paths imply small composition error when the growth factor is controlled.

\begin{proof}
Right translation gives $\frac{d}{dt}(\gamma(t)g_1)=(dR_{\gamma(t)g_1})_Ih(t)$. Thus, almost everywhere,
\begin{align*}
    \dot u(t)&=v_{X,h(t)}(u(t))
        +r_X(\gamma(t)g_1,\theta;h(t)),\\
    \dot w(t)&=v_{X,h(t)}(w(t))
        +r_X(\gamma(t),a_X(g_1,\theta);h(t)).
\end{align*}
The function $e(t)=\|u(t)-w(t)\|$ is absolutely continuous. The assumed growth bound and the triangle inequality imply
\[
    e'(t)\leq\lambda(t)e(t)
        +\|r_X(\gamma(t)g_1,\theta;h(t))\|
        +\|r_X(\gamma(t),a_X(g_1,\theta);h(t))\|
\]
almost everywhere. Multiplication by $e^{-\Lambda(t)}$ and integration give the asserted bound. Since $\gamma(0)=I$, we have
\[
    e(0)=\|a_X(I,a_X(g_1,\theta))-a_X(g_1,\theta)\|,
\]
while $\gamma(1)=g_2$ makes $e(1)$ the stated composition error.
\end{proof}

\paragraph{Loss variation along the same path.} In addition to the initial loss error, variation along the path has two sources: motion of the generators across loss level sets and deviation of the candidate map from the generator dynamics. The next bound separates these contributions. Put $z(t)=a_X(\gamma(t),\theta)$.

\begin{proposition}[Infinitesimal residuals control loss variation]
\label{prop:pathwise-loss-drift}
In the preceding setup,
\begin{align*}
    &\|L(a_X(g_2,\theta),X)-L(\theta,X)\|\\
    &\quad\leq\|L(a_X(I,\theta),X)-L(\theta,X)\|\\
    &\qquad+\int_0^1\Bigl(\|s_X(z(t);h(t))\|
      +\|D_\theta L|_{z(t),X}\|_{\mathrm{op}}
       \|r_X(\gamma(t),\theta;h(t))\|\Bigr)\,dt.
\end{align*}
\end{proposition}

Therefore, finite loss variation is controlled by the initial loss error, the infinitesimal loss residual, and the transport residual weighted by the loss derivative along the path. Only these trajectory-dependent quantities are needed, not a global bound on the loss derivative.

\begin{proof}
The definition of $r_X$ gives, almost everywhere,
\[
    \dot z(t)=v_{X,h(t)}(z(t))+r_X(\gamma(t),\theta;h(t)).
\]
By the chain rule,
\[
    \frac{d}{dt}L(z(t),X)
    =s_X(z(t);h(t))+
      D_\theta L\big|_{z(t),X}[r_X(\gamma(t),\theta;h(t))].
\]
The composition $t\mapsto L(z(t),X)$ is piecewise $C^1$, hence absolutely continuous. Integrating, taking norms, and adding the initial loss discrepancy proves the bound.
\end{proof}

The parameterization \eqref{eq:exact-identity-action} gives $u(0)=w(0)$ and $z(0)=\theta$, so both initial-error terms in the pathwise estimates vanish. We retain the general bounds below to cover other candidate maps.

\paragraph{Uniform bounds.} We state the full assumptions for Proposition~\ref{prop:uniform-finite-defects} before deriving it from the pathwise estimates. In the preceding setup, let $G$ be connected.

Use Euclidean norms on the parameter and loss spaces, and choose a norm on $\mathfrak g$. Suppose the following bounds hold for all $g\in G$, $h\in\mathfrak g$, and $\theta,\vartheta\in\mathbb R^p$, with nonnegative constants $\delta,\varepsilon_{\mathrm{act}},\varepsilon_{\mathrm{loss}},K,M$:
\begin{equation}
\label{eq:uniform-symmetry-assumptions}
\begin{aligned}
    \|a_X(I,\theta)-\theta\|&\leq\delta,\qquad
    \|r_X(g,\theta;h)\|\leq\varepsilon_{\mathrm{act}}\|h\|,\\
    \|s_X(\theta;h)\|&\leq\varepsilon_{\mathrm{loss}}\|h\|,\qquad
    \|D_\theta L|_{\theta,X}\|_{\mathrm{op}}\leq M,\\
    \|v_{X,h}(\theta)-v_{X,h}(\vartheta)\|
    &\leq K\|h\|\,\|\theta-\vartheta\|.
\end{aligned}
\end{equation}
For a piecewise $C^1$ path $\gamma\colon[0,1]\to G$ with finitely many pieces, define $h(t)$ by $\dot\gamma(t)=(dR_{\gamma(t)})_Ih(t)$ wherever the derivative exists, and write
\begin{equation}
\label{eq:group-path-length}
    \ell=\int_0^1\|h(t)\|\,dt,\qquad
    F_K(\ell)=
    \begin{cases}
        (e^{K\ell}-1)/K,&K>0,\\
        \ell,&K=0.
    \end{cases}
\end{equation}

\begin{manualproposition}{\ref{prop:uniform-finite-defects}}
Under \eqref{eq:uniform-symmetry-assumptions}, for every piecewise $C^1$ path $\gamma$ with finitely many pieces from $I$ to $g_2$, and every $g_1\in G$, $\theta\in\mathbb R^p$,
\begin{align*}
    \|a_X(g_2,a_X(g_1,\theta))-a_X(g_2g_1,\theta)\|
    &\leq\delta e^{K\ell}+2\varepsilon_{\mathrm{act}}F_K(\ell),\\
    \|L(a_X(g_2,\theta),X)-L(\theta,X)\|
    &\leq M\delta+
       (\varepsilon_{\mathrm{loss}}+M\varepsilon_{\mathrm{act}})\ell.
\end{align*}
\end{manualproposition}

\begin{proof}
For the composition bound, the Lipschitz assumption gives $\lambda(t)=K\|h(t)\|$ in Proposition~\ref{prop:pathwise-associativity}. The initial discrepancy is at most $\delta$, and the sum of the two transport residuals is at most $2\varepsilon_{\mathrm{act}}\|h(t)\|$. Set $q(t)=\int_0^t\|h(s)\|\,ds$. That proposition gives
\[
    \|u(1)-w(1)\|\leq\delta e^{K\ell}
    +2\varepsilon_{\mathrm{act}}\int_0^1
      e^{K(\ell-q(t))}\|h(t)\|\,dt.
\]
For $K>0$, integrating the derivative of $e^{K(\ell-q(t))}$ evaluates the integral as $(e^{K\ell}-1)/K$. For $K=0$, the integral equals $\ell$. This proves \eqref{eq:uniform-associativity-error}.

The derivative bound on $L(\cdot,X)$ makes it $M$-Lipschitz on $\mathbb R^p$, by integration along a straight line segment. Thus the initial loss discrepancy is at most $M\delta$. Substitution of the residual bounds into Proposition~\ref{prop:pathwise-loss-drift} gives \eqref{eq:uniform-loss-error}.
\end{proof}

\paragraph{Products of learned transformations.} A path that successively applies the learned generators requires bounds only in those directions. This gives the finite-error interpretation of the group samples used in Section~\ref{sec:discovery-objectives}.

Let $G$ be generated by $h_1,\ldots,h_r$ as in \eqref{eq:learned-generated-group}, with $\|h_i\|_F=1$. Retain the smooth candidate map and $C^1$ loss from the preceding setup. Suppose, for all $g\in G$, $\theta,\vartheta\in\mathbb R^p$, and $i$, that
\begin{equation}
\label{eq:generator-error-assumptions}
\begin{aligned}
    \|a_X(I,\theta)-\theta\|&\le\delta,
    &\|r_X(g,\theta;h_i)\|&\le\varepsilon_{\mathrm{act}},\\
    \|s_X(\theta;h_i)\|&\le\varepsilon_{\mathrm{loss}},
    &\|D_\theta L|_{\theta,X}\|_{\mathrm{op}}&\le M,\\
    \|v_{X,h_i}(\theta)-v_{X,h_i}(\vartheta)\|
    &\le K\|\theta-\vartheta\|.
\end{aligned}
\end{equation}
All constants are nonnegative, and identity may hold approximately.

\begin{corollary}[Errors along products of learned transformations]
\label{cor:learned-generator-word-errors}
Under \eqref{eq:generator-error-assumptions}, write
\[
    g_2=\exp(t_mh_{i_m})\cdots\exp(t_1h_{i_1}),
    \qquad \ell=\sum_{j=1}^m|t_j|.
\]
Then the composition and loss bounds \eqref{eq:uniform-associativity-error}--\eqref{eq:uniform-loss-error} hold for every $g_1\in G$ and $\theta\in\mathbb R^p$, with $F_K(\ell)=\int_0^\ell e^{Kt}\,dt$.
\end{corollary}

Each factor contributes $|t_j|$ to the path length. Thus the same bounds control mixed-generator compositions using residual bounds on the learned list.

\begin{proof}
For $m\ge1$, put $b_0=I$ and $b_j=\exp(t_jh_{i_j})b_{j-1}$. On $[(j-1)/m,j/m]$, define
\[
    \gamma(t)=\exp\bigl((mt-j+1)t_jh_{i_j}\bigr)b_{j-1}.
\]
This is a piecewise smooth path from $I$ to $g_2$, with right-logarithmic velocity $mt_jh_{i_j}$ on its $j$th interval and length $\ell$. Since the residuals and $v_{X,h}$ are linear in $h$, the transport residuals along this interval are bounded by $m|t_j|\varepsilon_{\mathrm{act}}$, the loss residual by $m|t_j|\varepsilon_{\mathrm{loss}}$, and the growth factor in Proposition~\ref{prop:pathwise-associativity} is $m|t_j|K$. That proposition therefore gives
\[
    \|u(1)-w(1)\|
    \le\delta e^{K\ell}
       +2\varepsilon_{\mathrm{act}}\int_0^\ell e^{K(\ell-s)}\,ds
    =\delta e^{K\ell}+2\varepsilon_{\mathrm{act}}F_K(\ell).
\]
The derivative bound makes $L(\cdot,X)$ $M$-Lipschitz, so the initial loss discrepancy is at most $M\delta$. Substituting the same interval bounds into Proposition~\ref{prop:pathwise-loss-drift} yields $M\delta+(\varepsilon_{\mathrm{loss}}+M\varepsilon_{\mathrm{act}})\ell$. For $m=0$, $g_2=I$ and both assertions follow directly from the identity-error bound.
\end{proof}

\paragraph{Inverse consistency.} A further consequence bounds the error when a transformation is followed by its proposed inverse. It is obtained by applying the composition bound to the pair $(g^{-1},g)$.

\begin{corollary}[Error in undoing a transformation]
\label{cor:uniform-inverse-error}
Under \eqref{eq:uniform-symmetry-assumptions}, let $\gamma$ be a piecewise $C^1$ path with finitely many pieces from $I$ to $g^{-1}$, and let $\ell$ be its length from \eqref{eq:group-path-length}. Then
\begin{equation}
\label{eq:uniform-inverse-error}
    \|a_X(g^{-1},a_X(g,\theta))-\theta\|
    \leq\delta+\delta e^{K\ell}
      +2\varepsilon_{\mathrm{act}}F_K(\ell).
\end{equation}
\end{corollary}

This controls how well $a_X(g^{-1},\cdot)$ undoes $a_X(g,\cdot)$; it does not assert that either candidate map is exactly invertible. Interchanging $g$ and $g^{-1}$ gives the reverse-composition bound, using a path from $I$ to $g$.

\begin{proof}
Apply \eqref{eq:uniform-associativity-error} with $(g_2,g_1)=(g^{-1},g)$, and then use $\|a_X(I,\theta)-\theta\|\leq\delta$ and the triangle inequality.
\end{proof}

\paragraph{Local bounds and empirical residuals.} For a particular comparison, the proof only needs transport bounds at the group--parameter pairs defining $u,w,z$, a growth bound between $u(t)$ and $w(t)$, and loss-residual and loss-derivative bounds along $z(t)$. To bound the initial loss error by $M\delta$, a loss-derivative bound along the segment between $\theta$ and $a_X(I,\theta)$ is enough. Thus the constants need not be finite on the entire parameter space. Small residuals at finitely many sampled points alone do not supply the trajectory bounds required by these results. Here approximate action refers to controlled identity, composition, and inverse-consistency errors on the specified domain. Closeness to an exact action and exact invertibility require additional assumptions.

% \newpage
\section{Building symmetries from known ones}
\label{appendix:building-symmetries-from-known-ones}

This section proves the lifting result and its width, depth, and batch-size consequences from Section~\ref{sec:building-symmetries-from-known-ones}. The small-subnetwork generation results are proved in Appendix~\ref{app:small-subnetworks-sufficient}.

\begin{manualproposition}{\ref{prop:symmetry-from-subnetworks}}
    Let $L\colon \Par \times \Data^d \to \R^d$ be a function, where the parameter space $\Par$ is a product space $\Par = \Par_1 \times \Par_2$, with spaces $\Par_1, \Par_2$.
    Suppose there exist functions 
    $h\colon \Par_1 \times \Data^d \to S$, 
    $f\colon \Theta_2 \times S \to T$, 
    and $j\colon (\Theta_1 \times T) \times  \Data^d \to \R^d$,
    such that for every $\theta=(\theta_1, \theta_2) \in \Par$ and $X \in \Data^d$, 
    $L(\theta, X) = j \big( (\theta_1, f \big(\theta_2, h(\theta_1, X) \big)), X \big)$.
    If $a\colon S \to (G \times \Par_2 \to \Par_2)$ is a $G$-symmetry of $f$, then there is an induced $G$-symmetry of $L$, $a'\colon \Data^d \to (G \times \Par \to \Par)$, defined by $a'_X(g, (\theta_1, \theta_2)) =  \big( \theta_1, a_{h(\theta_1, X)}(g, \theta_2) \big)$.
\end{manualproposition}

\begin{proof}
We need to show that $a'$ satisfies the identity and associative law of a group action and preserves $L$.

Since $a$ is a group action on $\Par_2$, it satisfies the identity axiom $a_{h(\theta_1, X)}(I, \theta_2) = \theta_2$.
Applying this in the definition of \(a'\), we get $a'_X(I, (\theta_1, \theta_2)) = (\theta_1, a_{h(\theta_1, X)}(I, \theta_2)) = (\theta_1, \theta_2)$.

Since $a$ is a group action on $\Par_2$, it satisfies the associative law $a_{h(\theta_1, X)}(g_2 g_1, \theta_2) = a_{h(\theta_1, X)}(g_2, a_{h(\theta_1, X)}(g_1, \theta_2))$, for all $g_1, g_2 \in G$.
It follows that $a'$ also satisfies the associative law: 
$a'_X(g_2 g_1, (\theta_1, \theta_2)) 
    = ( \theta_1, a_{h(\theta_1, X)}(g_2g_1, \theta_2) )
    = ( \theta_1, a_{h(\theta_1, X)}(g_2, a_{h(\theta_1, X)}(g_1, \theta_2)) )
    = a'_X(g_2, a'_X(g_1, (\theta_1, \theta_2)))$

Finally, since $a$ is a symmetry of $f$, we have $f(a_{h(\theta_1, X)}(g, \theta_2), h(\theta_1, X)) = f(\theta_2, h(\theta_1, X))$, for all $g \in G$.
It follows that $a'$ preserves the value of $L$:
$L(a'_X(g, \theta), X) 
    = j \big( (\theta_1, f \big(a_{h(\theta_1, X)}(g, \theta_2), h(\theta_1, X) \big)), X \big)
    = j \big( (\theta_1, f \big(\theta_2, h(\theta_1, X) \big) ), X \big)
    = L(\theta, X)$.
\end{proof}

\begin{manualcorollary}{\ref{cor:symmetry-in-wider-networks}}
        Consider the parameter space $\Par(m,h,n)=\R^{m\times h}\times\R^{h\times n}$ and data space $\Data(n,k)=\R^{n\times k}$. Let $\sigma$ apply the same row function $\sigma_{row}\colon\mathbb R^k\to\mathbb R^k$ at every hidden width. Define $L_{mnhk}\colon\Par(m,h,n)\times\Data(n,k)\to\R^{m\times k}$ by $L_{mnhk}((U,V),X)=U\sigma(VX)$. If $L_{mnhk}$ has a $G$-symmetry, then $L_{mnh'k}$ has a $G$-symmetry for every $h'>h$.
\end{manualcorollary}

\begin{proof}[Proof of Corollary~\ref{cor:symmetry-in-wider-networks}.]
Use the same row function $\sigma_{row}$ at every hidden width. Write $\theta_2=(U_{1:h},V_{1:h})$ for the first $h$ hidden units and $\theta_1=(U_{h+1:h'},V_{h+1:h'})$ for the remaining units. Here subscripts select columns of $U$ and rows of $V$. Since $\sigma$ acts independently on rows,
\[
    L_{mnh'k}((U,V),X)
    =L_{mnhk}(\theta_2,X)+L_{mn(h'-h)k}(\theta_1,X).
\]
Apply Proposition~\ref{prop:symmetry-from-subnetworks} with $S=\Data(n,k)$, $T=\R^{m\times k}$, and
\[
    h(\theta_1,X)=X,\qquad
    f(\theta_2,X)=L_{mnhk}(\theta_2,X),\qquad
    j((\theta_1,Y),X)=Y+L_{mn(h'-h)k}(\theta_1,X).
\]
If $a_X$ is the given symmetry of $L_{mnhk}$, the induced action is
\[
    a'_X(g,(\theta_1,\theta_2))
    =\bigl(\theta_1,a_X(g,\theta_2)\bigr).
\]
It applies the original symmetry to the first $h$ units and leaves all remaining units fixed.
\end{proof}

\begin{manualcorollary}{\ref{cor:symmetry-in-deeper-networks}}
    Let $\Par = \Par_1 \times ... \times \Par_l$ be a parameter space.
    Consider a list of spaces $V_0 = \Data^d$, $V_l = \R^d$, and $V_1$, ..., $V_{l-1}$.
    Let $L\colon \Par \times \Data^d \to \R^d$ be a function defined recursively by $\{L_i\}_{i=1}^l$ with $L_i\colon \Theta_i \times V_{i-1} \to V_{i}$, such that $L = \phi_l$ where $\phi_i = L_i(\theta_i, \phi_{i-1}) \in V_i$ and $\phi_0 = X$. 
    If for some $1 \leq i \leq l$, $L_i$ has a $G$-symmetry, then $L$ has a $G$-symmetry.
\end{manualcorollary}

\begin{proof}
Write $\theta_{-i}=(\theta_1,\ldots,\theta_{i-1},\theta_{i+1},\ldots,\theta_l)$ and $L_{j,\theta_j}(z)=L_j(\theta_j,z)$. With $\Par_1$ in Proposition~\ref{prop:symmetry-from-subnetworks} identified with the product of the parameter spaces other than $\Par_i$, and $\Par_2$ identified with $\Par_i$, define
\begin{align*}
    h(\theta_{-i},X)
    &=\bigl(L_{i-1,\theta_{i-1}}\circ\cdots\circ L_{1,\theta_1}\bigr)(X),\\
    f(\theta_i,z)&=L_{i,\theta_i}(z),\\
    j((\theta_{-i},z),X)
    &=\bigl(L_{l,\theta_l}\circ\cdots\circ L_{i+1,\theta_{i+1}}\bigr)(z).
\end{align*}
Empty compositions are identity maps, so $h(\theta_{-1},X)=X$ when $i=1$ and $j((\theta_{-l},z),X)=z$ when $i=l$. The recursive definition of $L$ gives
\[
    L(\theta,X)
    =j\bigl((\theta_{-i},f(\theta_i,h(\theta_{-i},X))),X\bigr).
\]
The prefix $h$ does not depend on $\theta_i$, and $f=L_i$ has the assumed $G$-symmetry. Proposition~\ref{prop:symmetry-from-subnetworks} therefore gives a $G$-symmetry of $L$.
\end{proof}

\begin{manualproposition}{\ref{prop:symmetry-with-smaller-data-batch}}
    Let $L_d\colon \Par \times \Data^d \to \R^d$ be a function that is applied pointwise on each of $d$ data points in a data batch.
    If $L_d$ admits a $G$-symmetry, then $L_{d'}$ admits a $G$-symmetry for all $d' < d$. 
\end{manualproposition}
\begin{proof}
Fix $1\leq d'<d$. For $X'=(x_1,\ldots,x_{d'})$, define the deterministic padding
\[
    \iota(X')=(x_1,\ldots,x_{d'},\underbrace{x_1,\ldots,x_1}_{d-d'\text{ entries}}).
\]
Let $a$ be the given $G$-symmetry of $L_d$, and set $a'_{X'}(g,\theta)=a_{\iota(X')}(g,\theta)$. For each fixed $X'$, the padded batch is fixed, so $a'_{X'}$ inherits identity and associativity from $a_{\iota(X')}$. In particular,
\[
    a'_{X'}(g_2,a'_{X'}(g_1,\theta))
    =a_{\iota(X')}(g_2g_1,\theta)
    =a'_{X'}(g_2g_1,\theta).
\]
Since $L_d$ and $L_{d'}$ apply the same function pointwise, projecting $L_d(\theta,\iota(X'))$ onto its first $d'$ components gives $L_{d'}(\theta,X')$. Applying this projection to the $L_d$-invariance identity proves $L_{d'}(a'_{X'}(g,\theta),X')=L_{d'}(\theta,X')$. Thus $a'$ is a $G$-symmetry of $L_{d'}$.
\end{proof}

\section{Proofs and further discussion of small-subnetwork symmetries}
\label{app:small-subnetworks-sufficient}

This section proves Theorem~\ref{thm:subnet-map-generation} and Corollary~\ref{cor:subnet-compensating-action}, and gives the explicit action and examples.

Throughout this section, we use the main-text setup and notations: $F_X(U,V)=U\sigma(VX)$, where $U\in\mathbb R^{m\times H}$, $V\in\mathbb R^{H\times n}$, $X\in\mathbb R^{n\times k}$, $H>k$, and $\sigma$ is smooth and row-wise. Let $Z=\sigma(VX)$. The set $B$ has $k$ indices, $C$ is its complement, and $Z_B$ is invertible on the nonempty domain $\Omega_B$ in \eqref{eq:subnet-regular-domain}. Subscripts on $U$ select columns; those on $V,Z$ select rows. 

\subsection{Decomposing an entire output-preserving transformation}
\label{app:small-subnetworks-generation}

\begin{manualtheorem}{\ref{thm:subnet-map-generation}}
Let $(\phi_t)_{t\in[0,1]}$ be a smooth family of diffeomorphisms of $\Omega_B$ satisfying $\phi_0=\mathrm{id}$ and $F_X\circ\phi_t=F_X$ for every $t\in[0,1]$.
Assume that all $\phi_t$ equal the identity outside the same compact set $K\subset\Omega_B$.
Then $\phi_1 = \psi_N \circ \ldots \circ \psi_1$ for finitely many output-preserving diffeomorphisms $\psi_i$, each changing at most $k+1$ hidden-unit blocks and preserving the combined output of those units on $X$.
\end{manualtheorem}

\begin{proof}
The output equation $Y=U_BZ_B+U_CZ_C$ uniquely determines
\begin{equation}
\label{eq:subnet-output-coordinates}
    U_B=(Y-U_CZ_C)Z_B^{-1}.
\end{equation}
Consequently,
\[
    \Phi(U,V)=(Y,V,U_C),\qquad Y=F_X(U,V),
\]
is a diffeomorphism from $\Omega_B$ onto
\[
    \mathcal O=\mathbb R^{m\times k}\times
    \{V:\det\sigma(V_BX)\neq0\}\times\mathbb R^{m\times(H-k)},
\]
with inverse given by \eqref{eq:subnet-output-coordinates}.

Flatten these coordinates as $(y,z)$ with $y=Y$ and $z=(V,U_C)\in\mathbb R^d$, where $d=Hn+m(H-k)$. The conjugate maps $\psi_t=\Phi\circ\phi_t\circ\Phi^{-1}$ preserve $y$ and equal identity outside a common compact subset of $\mathcal O$. Extend them by identity to the entire Euclidean coordinate space. These extensions are smooth diffeomorphisms because their supports stay inside $\mathcal O$.

First take an increment $\psi=\psi_b\circ\psi_a^{-1}$ sufficiently close to identity that $\|D(\psi-\mathrm{id})\|_\infty<1$. Write $\psi(y,z)=(y,f_1(y,z),\ldots,f_d(y,z))$ and define
\[
    Q_j(y,z)=(y,f_1(y,z),\ldots,f_j(y,z),z_{j+1},\ldots,z_d),
    \qquad Q_0=\mathrm{id},\quad Q_d=\psi.
\]
The derivative of $Q_j-\mathrm{id}$ is obtained by zeroing rows of $D(\psi-\mathrm{id})$, so its operator norm is also less than one. Each $Q_j$ is a global diffeomorphism: for any target $\eta$, the equation $x=\eta-(Q_j(x)-x)$ has a unique solution by the contraction theorem, and the nonsingular derivative gives a smooth inverse. Moreover, $Q_j$ fixes $\mathcal O^c$ pointwise. Bijectivity then forces $Q_j(\mathcal O)=\mathcal O$.

Set $T_j=Q_j\circ Q_{j-1}^{-1}$. The outputs $Q_j(x)$ and $Q_{j-1}(x)$ differ only in their $j$th $z$ coordinate. Hence $T_j$ changes only $z_j$, keeps $y$ fixed, and
\[
    \psi=T_d\circ\cdots\circ T_1.
\]
If $z_j$ is an entry of $V_i$, the inverse formula \eqref{eq:subnet-output-coordinates} changes only $V_i$ and $U_B$. If it is an entry of $U_i$ for $i\in C$, only $U_i$ and $U_B$ change. Each pulled-back factor therefore changes units in $B\cup\{i\}$, a set of at most $k+1$ units. Since it fixes all other units and preserves $F_X$, it preserves this subnetwork's combined output.

Finally, smoothness of $(\psi_t)$ and their common compact support imply that $\psi_b\circ\psi_a^{-1}$ is uniformly $C^1$-close to identity whenever $|b-a|$ is sufficiently small. Choose a finite partition of $[0,1]$ with all increments satisfying the preceding bound. Factor each increment and compose the factorizations. Conjugating back by $\Phi$ proves the result.
\end{proof}

\paragraph{Relation to decomposition results.} Factorization into maps preserving smaller sets of coordinates is related to fiber-preserving diffeomorphism decompositions \citep{haller2003smooth}. The global coordinates above give a direct proof for this network and identify the hidden-unit support of every factor. This differs from constructing a constant-output path between two parameter points \citep{nguyen2019connected}, and from complete ReLU equivalence through rewrites that may change architecture \citep{zhang2026complete}.

\subsection{An explicit compensating action}
\label{app:small-subnetworks-compensation}

A unit can change its incoming weights while a few other units cancel its output change. The output coordinates in the preceding proof give the following family without a compact-support restriction. For $P\in\mathbb R^{m\times(H-k)}$ and $Q\in\mathbb R^{(H-k)\times n}$, prescribe
\[
    V'_B=V_B,\quad V'_C=V_C+Q,\quad U'_C=U_C+P,\quad U'_B =\bigl[F_X(U,V)-U'_C\sigma(V'_CX)\bigr]
      \bigl[\sigma(V_BX)\bigr]^{-1}.
\]

\begin{manualcorollary}{\ref{cor:subnet-compensating-action}}
The maps $a_X^B((P,Q),(U,V))=(U',V')$ defined above form a smooth free action of the additive group $\mathbb R^{(H-k)(m+n)}$ on $\Omega_B$ and preserve $F_X$. This action is generated by commuting one-parameter subgroups, each changing at most $k+1$ hidden-unit blocks and preserving their combined output.
\end{manualcorollary}

This family of data-dependent symmetry fixes $V_B$ and is therefore not the entire class in Theorem~\ref{thm:subnet-map-generation}. 
Its construction follows from the output coordinates in that theorem's proof, rather than from applying the compact-support factorization to a translation.

\begin{proof}
In the coordinates $\Phi(U,V)=(Y,V,U_C)$ established in the preceding proof, the displayed transformation fixes $Y,V_B$ and translates $(U_C,V_C)$ by $(P,Q)$. Translations compose by addition, have inverses $(-P,-Q)$, act freely, and preserve $Y$. All group parameters are allowed because $V_B$, and hence the domain condition, remain unchanged. Substitution into \eqref{eq:subnet-output-coordinates} gives $U_B'$.

A translation of a single entry of $P$ or $Q$ modifies only one unit $j\in C$ and the output columns $U_B$. All other units remain fixed, and the total output is unchanged, so the combined contribution of $B\cup\{j\}$ is preserved. Coordinate translations commute and generate the additive group.
\end{proof}

\paragraph{Example: two sigmoid units.} For a single scalar input, let $F_x=u_1\sigma(v_1x)+u_2\sigma(v_2x)$ with logistic sigmoid $\sigma$. Since $\sigma(v_1x)>0$, any $q\in\mathbb R$ gives an exact transformation
\[
    v'_2=v_2+q,\qquad
    u'_1=u_1+u_2\frac{\sigma(v_2x)-\sigma((v_2+q)x)}{\sigma(v_1x)},
    \qquad u'_2=u_2,\quad v'_1=v_1.
\]
The second unit changes its incoming weight, and the first cancels its output change. 

\paragraph{Relation to earlier constructions.} \citet{zhao2023symmetries} give explicit nonlinear data-dependent group actions for single-input batches, while Corollary~\ref{cor:subnet-compensating-action} provides an explicit larger-batch action. 
The action on $\Omega_B$ is constructed by fixing $k$ units with independent batch features and using their output weights to compensate for changes in the remaining units. Keeping these units' incoming weights fixed preserves the invertibility needed for compensation, without requiring the activation to preserve rank everywhere. The subnetwork view makes both the construction and its $k+1$-unit support bound explicit. 
Output compensation also appears in constant-output path constructions \citep{nguyen2019connected}, but the transformations here additionally have a specified group structure.

\paragraph{Application to gated feedforward blocks.} The proofs use only a linear output readout and independence of each feature row from the other units' incoming parameters. They therefore apply to the smooth gated block
\[
    F_X=U\bigl[\sigma(W_gX)\odot(W_uX)\bigr],
\]
by treating one row of $W_g$, one row of $W_u$, and the matching column of $U$ as one unit \citep{shazeer2020glu}. If $k$ feature rows are independent, the same construction has support at most $k+1$ units, with an explicit additive subgroup of dimension $(H-k)(m+2n)$. An unchanged residual branch can be added without affecting either output-preservation argument. In a token-processing block, $k$ counts all protected token positions used at this boundary; the rank condition must hold for that complete feature matrix.

\section{Experimental setup and validation}
\label{app:experimental-protocol}
\label{app:experimental-details}

The one-generator small-network experiments use five optimization seeds, $101$--$105$, using independently drawn parameter perturbations and transformations for training, checkpoint selection, recipe selection, and final testing. The task seed is $31415$. Replications therefore measure optimization variability for fixed base networks and conditioning instances. The higher-dimensional protocols are given in Appendix~\ref{app:frontier-recovery}, and the pretrained-model protocol is given in Appendix~\ref{app:experimental-pythia-details}.

\subsection{Implementation of the discovery objectives}
\label{app:discovery-objectives}

This subsection gives the group, sampling, normalization, and regularization details for the direct-map method in Section~\ref{sec:automatic-discovery}. In subnetwork discovery, unchanged parameters form part of the fixed context. Affine-field extraction is described in Appendix~\ref{app:frontier-recovery}.

\paragraph{Learning the candidate group.} We fix the matrix size $s$ and generator count $r$, and learn normalized matrices $h_1,\ldots,h_r$ shared across data batches. Their finite exponential products define $G$ in \eqref{eq:learned-generated-group}. This subgroup has a connected immersed Lie-group structure. Its Lie algebra is the smallest vector space containing the $h_i$ and closed under matrix commutators. We use this intrinsic smooth structure when $G$ is not closed in $\mathrm{GL}(s,\mathbb R)$. The action network is smooth in its ambient matrix input, so its restriction to $G\times\mathbb R^p$ is smooth. Symmetry conditions are imposed on this restriction.

The generator count $r$ can be smaller than the resulting algebra dimension. For example, $E_{12}$ and $E_{21}$ generate $[E_{12},E_{21}]=\operatorname{diag}(1,-1)$ and hence a three-dimensional algebra. Here $E_{ij}$ is the $2\times2$ matrix with a single one in position $(i,j)$. Proposition~\ref{prop:learned-generators-suffice} applies to a generating list of any such group.

\paragraph{Exact identity.} Equation~\eqref{eq:exact-identity-action} represents every smooth identity-preserving map at the unrestricted function level by taking $b_X(g,\theta)=a_X(g,\theta)-\theta$. Both evaluations use the same deterministic network and context. Compute their difference before adding $\theta$. For a restricted parameter domain, the architecture must additionally keep its output in that domain.

\paragraph{Sampling group elements.} Let $\mu$ sample the original parameter--batch pairs $(\theta,X)$. Sample a product length, generator indices, and real coefficients, and multiply their matrix exponentials to obtain $g$. Let $\pi$ denote this distribution at the current learned matrices and let $\nu$ return each $h_i$ with probability $1/r$. Include negative coefficients, different orders, and products of several generators. The action network receives the product matrix, so different expressions for the same matrix give the same network input. Keep the coefficient distribution fixed within each training stage. We begin with short products at moderate coefficient magnitudes and expand the range as training progresses.

For a supplied connected matrix group $G$, fix $h_1,\ldots,h_r$ to a Frobenius-orthonormal basis of $\mathfrak g$, choose $\pi$ on $G$, and let $\nu$ be uniform on this basis. Here $r=\dim\mathfrak g$ and only the action map is learned.

\paragraph{Sampling parameter points.} The differential losses use an equal mixture of original pairs from $\mu$ and pairs $(a_X(q,\theta),X)$ obtained with an independently sampled product $q$. This exposes training to points visited by the candidate transformations; the pathwise bounds require residual control along the full evaluated paths. When constructing these additional sample points, detach them from the sampling computation. Derivatives of the residuals are then taken at the selected points. The regularizers always use the original distribution $\mu$ so that their motion scale remains fixed.

The finite composition and finite preservation terms use the same mixture of original and transformed starting points. For $\mathcal L_{\mathrm{comp}}$, use a mixture of independent products $(g_1,g_2)$ and inverse pairs $(g_1,g_1^{-1})$. The batch $X$ stays fixed within each composition. Both sides of the composition residual are differentiated through in full, including the inner transformed parameter. This finite term adds direct supervision of composition over the sampled range. The infinitesimal-only variant omits both finite terms.

\paragraph{Automatic differentiation.} Compute directional derivatives using Jacobian--vector products. In $D_ga_X|_{g,\theta}(hg)$, differentiate only the group input and hold the parameter input fixed. For the group derivative at $I$, the identity matrix in $b_X(I,\theta)$ remains a constant. Retain the dependence of both the direction $hg$ and the evaluation point $a_X(g,\theta)$ when backpropagating the complete transport loss. Hold the sampled indices and coefficients fixed and differentiate through the matrix exponentials and action network. Matrix-valued outputs are flattened before taking Euclidean norms.

\paragraph{Preservation target.} The target function is held fixed during discovery. We retain every network or subnetwork output component inside the squared norm. For ReLU targets, derivatives are interpreted away from activation boundaries, and finite preservation is checked separately. Exact reconstruction of ReLU rescalings uses positive homogeneity, including at activation boundaries (Appendix~\ref{app:frontier-recovery}). The finite-error theorem in Section~\ref{sec:approximate-symmetry} assumes a $C^1$ target.

\paragraph{Finite preservation and objective variants.} The finite-preservation term is
\begin{equation}
\label{eq:discovery-finite-preservation-loss}
    \mathcal L_{\mathrm{finite}}
    =\mathbb E_{X,\theta,g}
      \|L(a_X(g,\theta),X)-L(\theta,X)\|_2^2.
\end{equation}
The base hybrid objective is \eqref{eq:objective-weighted} with $\gamma_2>0$. The finite comparator removes $\mathcal L_{\mathrm{invariance}}$ and $\mathcal L_{\mathrm{assoc}}$ and includes a positive multiple of $\mathcal L_{\mathrm{finite}}$, retaining finite composition and the same regularizers. The infinitesimal-only comparator sets $\gamma_2=0$ and omits $\mathcal L_{\mathrm{finite}}$. Stronger hybrid recipes add $\mathcal L_{\mathrm{finite}}$ to the base objective. Each objective is tuned separately on validation data.

\paragraph{Fixed normalization.} Section~\ref{sec:automatic-discovery} writes the losses before task-specific normalization. In implementation, preservation residuals are divided by a fixed output scale $s_F$, and transport and composition residuals in physical parameter coordinates are divided by $s_\theta\beta$, before squaring. The regularizers use $v_i/s_\theta$ in place of $v_i$, so $\beta$ specifies a field magnitude in units of the fixed parameter scale $s_\theta$. These positive constants leave the zero-residual conditions unchanged. They are fixed before optimization; their definitions for small networks and Pythia are given below and in Appendix~\ref{app:experimental-pythia-details}. Reduced-coordinate actions use the induced physical parameter norm; the Pythia coordinates are isometric for these increments. The motion-relative composition and inverse errors used in evaluation have sample-dependent denominators and are distinct from these fixed training normalizations.

\paragraph{Interpreting the regularizers.} The scale penalty targets mean squared field magnitude $\beta^2$, while the diversity penalty discourages duplicate, opposite, or correlated fields. Both use uncentered moments so that constant translation fields count as nontrivial motion. They allow a field to vanish at individual parameter points. Pointwise independence is evaluated separately through the rank of $[v_1(\theta,X)\ \cdots\ v_r(\theta,X)]$.

Use nonzero initialization for the generators and the group-dependent part of $b_X$, since the squared scale penalty has zero gradient at a fully collapsed parameterization. We evaluate actual finite parameter displacements as well as generator norms. The regularization weights and reference scale are selected on validation data, balancing invariance, composition, nontrivial motion, and redundancy.

\paragraph{Examining the learned structure.} Repeatedly adjoining commutators and taking their linear span determines the generated Lie algebra in exact arithmetic. Numerical dimension estimates should report the rank tolerance and its sensitivity. For an exact action, \eqref{eq:generator-bracket-coordinates} requires the same bracket relations in the induced parameter fields. Directions whose fields vanish identically belong to the kernel of the infinitesimal action. Thus the matrix algebra, its effective action, and the rank of its fields at a parameter point describe different aspects of the learned symmetry. A nonzero derivative with respect to $X$ shows dependence of the parameterization on data. To examine dependence of the preserved function, evaluate the fixed transformation $a_X(g,\theta)$ on a separate batch $X'$.

Training with bounded products and finitely many samples is assessed through held-out errors, including longer compositions. The pathwise bounds in Appendix~\ref{app:associativity-stability} require residual control along the evaluated paths.

\subsection{One-generator experimental setup}
\label{app:one-generator-setup}

\paragraph{Models and parameter sampling.} Linear discovery uses $2\times2$ factors near identity, rejecting numerically singular samples. ReLU recovery uses a bias-free $(2,4,2)$ network. Focused sigmoid discovery preserves $U_B\sigma(V_BX)+u_c\sigma(v_cX)$, with $V_B,u_c$ fixed and $(U_B,v_c)$ transformed. The single-input case has $m=n=k=1$; the two-input case has $m=n=k=2$. Compensating units are selected by pivoted QR during setup and their indices remain fixed. The selected feature matrix must have condition number at most $10^4$. The tanh and GELU targets are bias-free $(4,16,16,2)$ students trained for $3{,}000$ steps against a fixed random teacher on Gaussian inputs. We transform the first and last matrices and hold the middle matrix fixed. Gaussian parameter perturbations have coordinate standard deviation $0.20$ for linear and scalar sigmoid, $0.30$ for ReLU recovery, $0.15$ for two-input sigmoid, $0.05$ for the deeper networks, and $0.20$ for host lifting. Data-dependent fits use a fixed conditioning batch.

\paragraph{Training.} The action network has three width-$256$ SiLU hidden layers and exact identity by construction. We learn one Frobenius-normalized matrix generator per action. Training uses Adam in float32, gradient clipping at $10$, and equal proportions of original and transformed parameter locations for the differential residuals. One quarter of the composition pairs are inverse pairs. Field-scale regularization is active and sparsity is disabled. Each objective has a separate validation search over learning rates, batch sizes, and residual weights. Initial candidates train for $5{,}000$ steps, and the strongest candidates continue to $15{,}000$ total steps before selecting a recipe for fresh-seed replication. The selected learning rates are $10^{-3}$ or $3\times10^{-4}$, with batches of $16$ or $32$. The scalar-sigmoid hybrid recipe uses coefficient radius $0.8$ and at most six factors; the other main-text hybrid recipes use radius $0.5$ and at most three factors. The finite and infinitesimal comparators retain the same action architecture and field regularizers. Selected recipes can differ in batch size and transformation range, so the comparisons assess tuned objectives at equal final step budgets.

\paragraph{Finite-action tests.} We keep the conditioning batch and parameter mask fixed within each comparison. Group samples are signed products $\exp(t_mh_{i_m})\cdots\exp(t_1h_{i_1})$. Coefficient radius $\rho$ is the upper bound on their total absolute coefficient length $\sum_j|t_j|$, sampled uniformly on $[0.2\rho,\rho]$. It controls the group input rather than fixing the resulting parameter displacement. The test grid contains $\rho\in\{0.1,0.3,0.5,0.8,1.2\}$ and one, two, four, or eight factors, with $512$ samples per cell. Tables use the declared cell $\rho=0.5$ with one factor. Longer signed products can have smaller displacement through cancellation.

Let $s_\theta$ be the root-mean-square parameter-coordinate magnitude and $s_F$ the root-mean-square output norm, estimated from $64$ calibration samples and fixed before optimization. Write $\theta_1=a_X(g_1,\theta)$ and $\theta_{12}=a_X(g_2,\theta_1)$. We measure
\[
\begin{aligned}
E_{\mathrm{out}}&=\frac{\|F_X(\theta_1)-F_X(\theta)\|}{s_F},\\
E_{\mathrm{comp}}&=\frac{\|\theta_{12}-a_X(g_2g_1,\theta)\|}{\|\theta_1-\theta\|+\|\theta_{12}-\theta_1\|},\\
E_{\mathrm{inv}}&=\frac{\|a_X(g_1^{-1},\theta_1)-\theta\|}{\|\theta_1-\theta\|}.
\end{aligned}
\]
For example, a composition error of $10^{-3}$ is a discrepancy of $0.1\%$ of the two-step displacement. An output error of $10^{-3}$ is $0.1\%$ of the fixed calibration scale $s_F$, rather than a percentage change in prediction accuracy. The implementation adds $10^{-12}$ to denominators after scaling parameter norms by $s_\theta$. Transport is normalized by the corresponding field norm. Subdivision compares one exponential transformation with four applications at one quarter of its coefficient, normalized by the one-step motion. Cancellation is $\|\Delta Y_B+\Delta Y_C\|/(\|\Delta Y_B\|+\|\Delta Y_C\|+10^{-12})$. Each one-generator small-network table entry aggregates the within-seed statistic by its median across all five seeds.

\paragraph{Joint numerical tolerances.} The error and motion requirements are fixed before final testing. The $95$th-percentile output error must be at most $10^{-3}$; composition, inverse, and subdivision errors must be at most $10^{-2}$; and relative transport error must be at most $2\times10^{-2}$. These requirements apply to both one- and two-factor tests at radius $0.5$. One-factor median motion must be at least $0.05s_\theta$. Compensation additionally requires median moving-output change of at least $0.005s_F$ and cancellation-error $95$th percentile at most $10^{-2}$. Columns headed ``Fits'' count how many fitted actions satisfy every requirement. The reported errors and motion include every seed, regardless of that count. These tolerances summarize the sampled tests and are not global guarantees.

Table~\ref{tab:experimental-main} reports the complete one-generator results, including all seeds in each setting.

\begin{table}[ht]
\centering
\small
\setlength{\tabcolsep}{4pt}
\begin{tabular}{lrrrrr}
\toprule
Setting & Motion (\%) & Output & Composition & Inverse & Fits \\
\midrule
Linear & 5.25 & $4.73\!\times\!10^{-4}$ & $3.68\!\times\!10^{-3}$ & $3.27\!\times\!10^{-3}$ & 5/5 \\
ReLU & 3.74 & $8.70\!\times\!10^{-4}$ & $3.07\!\times\!10^{-3}$ & $6.14\!\times\!10^{-3}$ & 3/5 \\
Sigmoid, $k=1$ & 11.95 & $5.23\!\times\!10^{-5}$ & $3.71\!\times\!10^{-4}$ & $5.16\!\times\!10^{-4}$ & 5/5 \\
Sigmoid, $k=2$ & 6.35 & $1.46\!\times\!10^{-4}$ & $1.48\!\times\!10^{-3}$ & $2.16\!\times\!10^{-3}$ & 4/5 \\
Tanh, separated layers & 1.50 & $9.57\!\times\!10^{-4}$ & $3.29\!\times\!10^{-3}$ & $3.53\!\times\!10^{-3}$ & 2/5 \\
GELU, separated layers & 1.50 & $1.07\!\times\!10^{-3}$ & $4.41\!\times\!10^{-3}$ & $3.77\!\times\!10^{-3}$ & 2/5 \\
ReLU subnetwork & 4.81 & $7.51\!\times\!10^{-4}$ & $4.29\!\times\!10^{-3}$ & $6.62\!\times\!10^{-3}$ & 5/5 \\
\bottomrule
\end{tabular}
\caption{Learned one-generator approximate actions in the small-network experiments. Errors are medians across five seeds of within-seed $95$th percentiles over $512$ samples at coefficient radius $0.5$ and one factor. Output error uses a fixed calibration scale; composition and inverse errors are relative to attempted motion. Motion is the median relative change in the transformed parameter block. The final column counts fits satisfying every tolerance in Appendix~\ref{app:experimental-protocol}, including two-factor tests, transport, subdivision, and nontrivial motion.}
\label{tab:experimental-main}
\end{table}

\paragraph{Generated curves.} We additionally evaluate $\theta(t)=a_X(\exp(th),\theta_0)$ at $61$ equally spaced times in $[-0.5,0.5]$, using $32$ fresh parameter points per seed in float64. Each trajectory keeps $\theta_0$, $X$, and $h$ fixed. The horizontal coordinate $t$ is the signed exponential coefficient and $t=0$ is the original parameter point. The plotted output drift is $\|F_X(\theta(t))-F_X(\theta_0)\|_F/s_F$. Its random reference is a fixed parameter translation $\theta_0+t z$, with direction drawn once and scale calibrated to match root-mean-square physical motion at $t=0.3$ on separate validation parameters. Aggregate curves retain all five seeds. Illustrations use seed $101$ and the first sampled parameter point, selected before evaluating these curves. These are descriptive trajectory evaluations; the randomized finite-action tests and their joint numerical criterion are unchanged.

\paragraph{Main-text curve summaries.} In Figure~\ref{fig:experimental-compensation}(a), each unit's contribution is $Y_j=u_j\sigma(v_jX)$, and its signed change is $\Delta Y_j=Y_j(\theta(t))-Y_j(\theta_0)$. These changes and their sum are divided by $s_F$. This panel uses the prespecified illustration above. Panel (b) takes the median output drift over the $32$ initial parameter points separately in each seed, then the median across seeds. In panel (c), the horizontal coordinate is the seed median raw parameter displacement and the vertical coordinate is the seed $95$th-percentile full-network output error; each is then aggregated by its median across five seeds. Shaded regions span the corresponding seed summaries. The host comparison uses identical complete network parameters and inputs for the two methods, as detailed in Appendix~\ref{app:experimental-subnetwork}.

\FloatBarrier

\section{Objective comparisons and batch-specific preservation}
\label{app:experimental-comparisons}

Table~\ref{tab:experimental-objectives} compares the objectives on the one-generator small-network settings. The hybrid objective gives lower median preservation, composition, and inverse errors than the finite comparator in each paired setting. On scalar sigmoid, the infinitesimal-only objective also produces accurate preserving actions. Adding finite composition to the tuned recipe reduces its median composition error from $2.42\times10^{-3}$ to $3.71\times10^{-4}$.

\begin{table}[ht]
\centering\small
\setlength{\tabcolsep}{3.5pt}
\begin{tabular}{llrrrrr}
\toprule
Setting & Objective & Output & Comp. & Inverse & Transport & Fits \\
\midrule
Linear & Hybrid & $4.73\!\times\!10^{-4}$ & $3.68\!\times\!10^{-3}$ & $3.27\!\times\!10^{-3}$ & $3.92\!\times\!10^{-3}$ & 5/5 \\
Linear & Finite & $1.31\!\times\!10^{-3}$ & $7.33\!\times\!10^{-3}$ & $7.84\!\times\!10^{-3}$ & $1.27\!\times\!10^{-2}$ & 0/5 \\
ReLU & Hybrid & $8.70\!\times\!10^{-4}$ & $3.07\!\times\!10^{-3}$ & $6.14\!\times\!10^{-3}$ & $6.55\!\times\!10^{-3}$ & 3/5 \\
ReLU & Finite & $2.56\!\times\!10^{-3}$ & $1.10\!\times\!10^{-2}$ & $1.50\!\times\!10^{-2}$ & $1.87\!\times\!10^{-2}$ & 0/5 \\
Sigmoid, $k=1$ & Hybrid & $5.23\!\times\!10^{-5}$ & $3.71\!\times\!10^{-4}$ & $5.16\!\times\!10^{-4}$ & $8.88\!\times\!10^{-4}$ & 5/5 \\
Sigmoid, $k=1$ & Finite & $1.07\!\times\!10^{-4}$ & $3.06\!\times\!10^{-3}$ & $3.60\!\times\!10^{-3}$ & $4.24\!\times\!10^{-3}$ & 5/5 \\
Sigmoid, $k=1$ & Infinitesimal & $5.16\!\times\!10^{-5}$ & $2.42\!\times\!10^{-3}$ & $1.43\!\times\!10^{-3}$ & $1.45\!\times\!10^{-3}$ & 5/5 \\
Sigmoid, $k=2$ & Hybrid & $1.46\!\times\!10^{-4}$ & $1.48\!\times\!10^{-3}$ & $2.16\!\times\!10^{-3}$ & $2.03\!\times\!10^{-3}$ & 4/5 \\
Sigmoid, $k=2$ & Finite & $2.50\!\times\!10^{-4}$ & $5.43\!\times\!10^{-3}$ & $4.15\!\times\!10^{-3}$ & $9.40\!\times\!10^{-3}$ & 3/5 \\
Sigmoid, $k=2$ & Infinitesimal & $1.71\!\times\!10^{-4}$ & $1.70\!\times\!10^{-3}$ & $1.66\!\times\!10^{-3}$ & $2.36\!\times\!10^{-3}$ & 3/5 \\
Tanh & Hybrid & $9.57\!\times\!10^{-4}$ & $3.29\!\times\!10^{-3}$ & $3.53\!\times\!10^{-3}$ & $5.23\!\times\!10^{-3}$ & 2/5 \\
Tanh & Finite & $2.40\!\times\!10^{-3}$ & $2.36\!\times\!10^{-2}$ & $9.14\!\times\!10^{-3}$ & $3.27\!\times\!10^{-2}$ & 0/5 \\
GELU & Hybrid & $1.07\!\times\!10^{-3}$ & $4.41\!\times\!10^{-3}$ & $3.77\!\times\!10^{-3}$ & $7.52\!\times\!10^{-3}$ & 2/5 \\
GELU & Finite & $2.75\!\times\!10^{-3}$ & $2.53\!\times\!10^{-2}$ & $9.75\!\times\!10^{-3}$ & $3.69\!\times\!10^{-2}$ & 0/5 \\
Local ReLU host & Hybrid & $7.51\!\times\!10^{-4}$ & $4.29\!\times\!10^{-3}$ & $6.62\!\times\!10^{-3}$ & $7.66\!\times\!10^{-3}$ & 5/5 \\
\bottomrule
\end{tabular}
\caption{Objective comparisons for the one-generator small-network settings. Errors are medians across all five seeds of their within-seed $95$th percentiles at radius $0.5$ and one factor. The joint numerical criterion additionally includes the two-factor cell, subdivision, and motion. Each objective is tuned independently.}
\label{tab:experimental-objectives}
\end{table}

\subsection{Transformations across non-contiguous layers}
\label{app:experimental-separated-layers}
\label{sec:experimental-separated-layers}

We also learn one-generator transformations in tanh and GELU networks with widths $(4,16,16,2)$, trained on synthetic teacher regression. Discovery changes the first and last weight matrices while keeping the middle layer fixed, giving a $96$-parameter search space. With $24$ protected inputs, median relative parameter motion is approximately $1.5\%$. Two of five fits for each activation satisfy all the joint numerical tolerances in Appendix~\ref{app:experimental-protocol}; all five are retained in the tables and figures.

To examine finite preservation, we evaluate the output along curves $\theta(t)=a_X(\exp(th),\theta_0)$ with $\theta_0$, $X$, and $h$ fixed along each curve. Compared with random parameter translations calibrated to comparable motion, the learned curves have much smaller output variation (Figure~\ref{fig:experimental-orbits}(b)). The hybrid objective also reduces the median composition error from $2.36\times10^{-2}$ to $3.29\times10^{-3}$ for tanh and from $2.53\times10^{-2}$ to $4.41\times10^{-3}$ for GELU. These measurements describe coordinated approximate transformations on the tested parameter neighborhoods. The first/last-layer support is specified before training.

\subsection{Preservation on conditioning and fresh batches}

To examine batch-specific preservation, we first construct $\theta'=a_X(g,\theta)$ using the conditioning batch, and then evaluate these same weights on an independently sampled $X'$. Table~\ref{tab:experimental-data-dependence} shows small errors on $X$ and larger errors on $X'$. These measurements use $128$ parameter samples and one- to three-factor products at radius $0.5$. They describe one frozen transformation evaluated on different inputs.

\begin{table}[ht]
\centering\small
\begin{tabular}{lrr}
\toprule
Setting & Conditioning batch & Fresh batch, same weights \\
\midrule
Sigmoid, $k=1$ & $4.46\!\times\!10^{-5}$ & $1.11\!\times\!10^{-1}$ \\
Sigmoid, $k=2$ & $1.14\!\times\!10^{-4}$ & $1.13\!\times\!10^{-2}$ \\
Tanh & $8.06\!\times\!10^{-4}$ & $3.58\!\times\!10^{-3}$ \\
GELU & $9.23\!\times\!10^{-4}$ & $5.48\!\times\!10^{-3}$ \\
\bottomrule
\end{tabular}
\caption{Preservation on conditioning and fresh batches using the same transformed weights. Entries are medians across five seeds of within-seed $95$th-percentile errors, using the original calibration scale. Both columns use exactly the same frozen weights.}
\label{tab:experimental-data-dependence}
\end{table}

\begin{figure}[t]
\centering
\begin{minipage}{0.485\linewidth}\centering
\includegraphics[width=\linewidth]{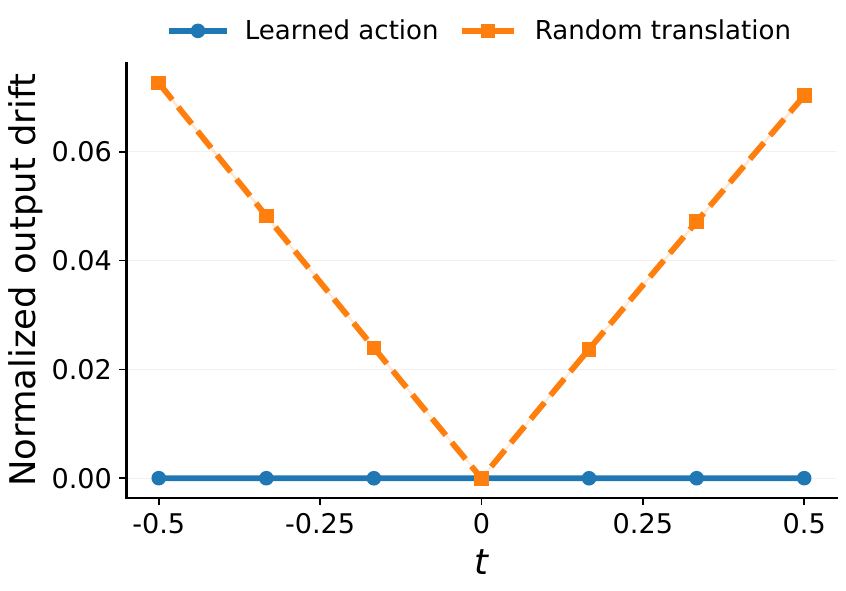}
\small (a) One protected input
\end{minipage}\hfill
\begin{minipage}{0.485\linewidth}\centering
\includegraphics[width=\linewidth]{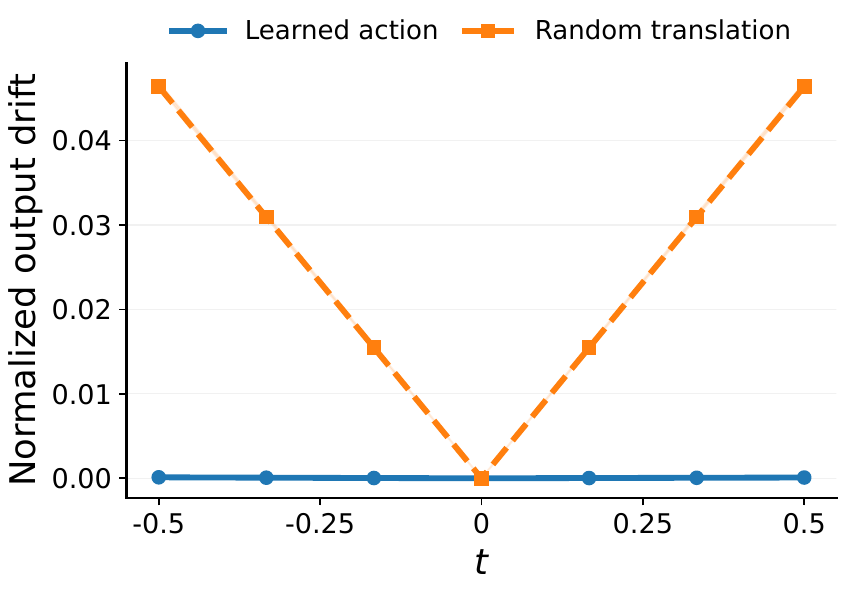}
\small (b) Two protected inputs
\end{minipage}
\caption{Output variation along sigmoid trajectories with one or two protected inputs. The horizontal axis is the signed coefficient $t$ in $\theta(t)=a_X(\exp(th),\theta_0)$. Normalized output drift is $\|F_X(\theta(t))-F_X(\theta_0)\|_F/s_F$, with $s_F$ fixed before training. Learned action follows this curve; random translation follows $\theta_0+t z$ with $z$ fixed and scaled on validation parameters to comparable motion at $t=0.3$. Curves take medians over $32$ initial points, then across all five seeds; shading shows the minimum and maximum seed medians.}
\label{fig:additional-nonlinear-orbits}
\end{figure}

\FloatBarrier

\section{Recovered directions and higher-dimensional actions}
\label{app:experimental-structure}
\label{app:frontier-recovery}

\subsection{One-parameter reference directions}

We compare the induced field $v_{X,h}(\theta)=D_ga_X|_{I,\theta}[h]$ with reference generator values at the same parameter points. The containment error is the normalized distance of the learned tangent span from the reference span. Table~\ref{tab:experimental-reference} uses $24$ parameter samples per seed at a fixed context and relative rank tolerance $10^{-4}$. The one-generator matrix algebra is one-dimensional and abelian by construction. Its induced field has sampled rank one in these fits. The small containment errors demonstrate agreement with the reference directions. Each experiment targets a single one-parameter subgroup.

\begin{table}[ht]
\centering\small
\begin{tabular}{lrrr}
\toprule
Setting & Field dim. & Orbit rank & Reference-span error \\
\midrule
Linear & 1 & 1 & $1.92\!\times\!10^{-3}$ \\
ReLU & 1 & 1 & $7.14\!\times\!10^{-3}$ \\
Sigmoid, $k=1$ & 1 & 1 & $1.35\!\times\!10^{-4}$ \\
Sigmoid, $k=2$ & 1 & 1 & $9.93\!\times\!10^{-4}$ \\
Local ReLU host & 1 & 1 & $3.91\!\times\!10^{-3}$ \\
\bottomrule
\end{tabular}
\caption{Induced-field agreement for the reported one-generator actions. Containment entries are medians across five seeds of mean normalized reference-span distances. Field dimension and sampled orbit rank equal one in each seed.}
\label{tab:experimental-reference}
\end{table}

Figure~\ref{fig:experimental-reference-fields} visualizes the reference agreement directly. Each point compares one coordinate of the learned field with the same coordinate of its pointwise projection onto the analytic reference span. This projection allows the coefficients to depend on the parameter point. The plot therefore illustrates tangent-direction agreement, while the finite-action tests evaluate composition and preservation separately.

\begin{figure}[t]
\centering
\begin{minipage}{0.485\linewidth}\centering
\includegraphics[width=\linewidth]{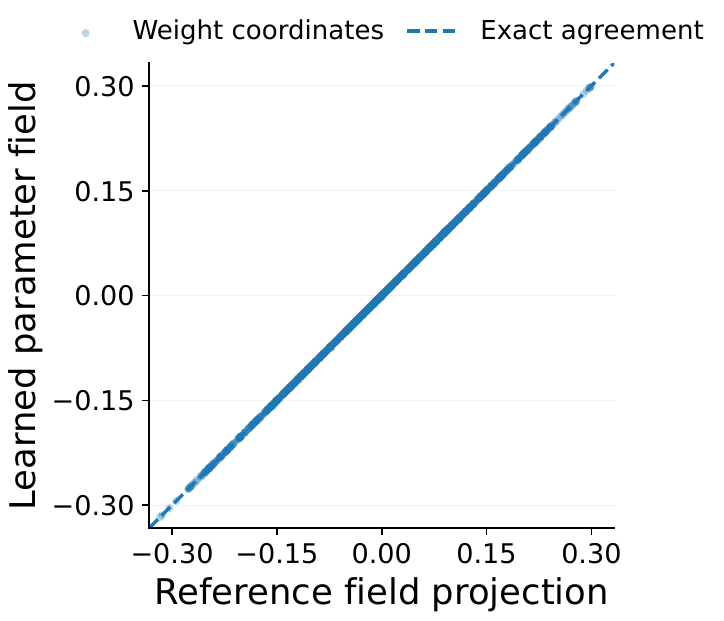}
\small (a) Linear network
\end{minipage}\hfill
\begin{minipage}{0.485\linewidth}\centering
\includegraphics[width=\linewidth]{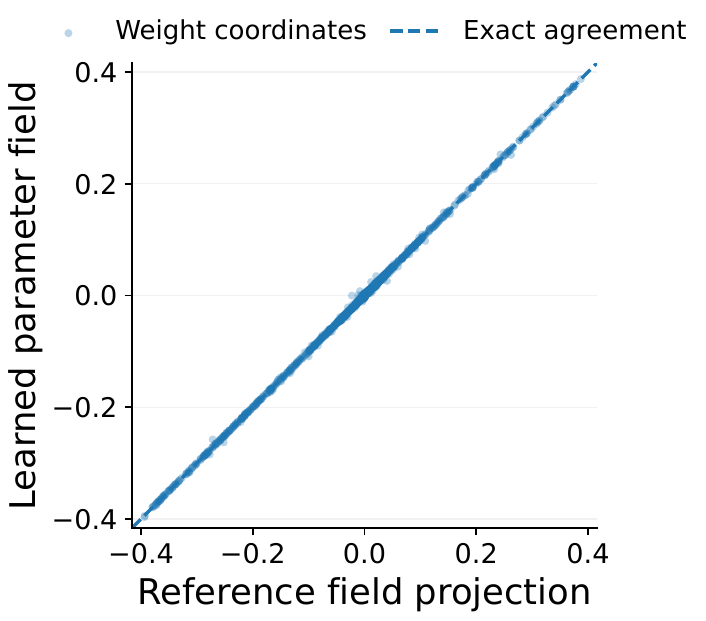}
\small (b) ReLU network
\end{minipage}
\caption{Induced parameter fields and known symmetry directions. Each dot (weight coordinates in the legend) represents one coordinate of $v_{X,h}(\theta)=D_ga_X|_{I,\theta}[h]$, the instantaneous parameter change per unit coefficient. Its horizontal value is that coordinate after orthogonal projection onto the reference tangent span at the same $\theta$; its vertical value is the learned field coordinate. The dashed exact-agreement line is $y=x$. All coordinates from $32$ parameter configurations and five seeds are shown. The projection is pointwise, so the plots measure tangent-direction agreement.}
\label{fig:experimental-reference-fields}
\end{figure}

\subsection{Affine-field extraction and exact verification}

We search for parameter fields $v(\theta)=A\theta+b$ without providing a reference basis or its dimension. For each network, we sample $512$ complete parameter vectors and four Gaussian inputs per vector. An entry in a weight matrix with input width $d$ has distribution $\mathcal N(0,1/d)$. With $J= D_\theta F_X(\theta)$ and $\bar\theta=(\theta^\top,1)^\top$, the constraints are
\[
    J[A\; b]\bar\theta=0.
\]
They are linear in the $p(p+1)$ unknown coefficients of $[A\;b]$. We stack the componentwise equations, normalize their rows, and take the right singular vectors below $10^{-8}$ times the largest singular value. This is an affine extension of the linear-generator extraction used by \citet{moskalev2022liegg}. We repeat extraction with seeds $101$--$103$. The recovered nullities are unchanged at relative thresholds $10^{-4}$, $10^{-6}$, and $10^{-8}$ in every run.

The affine-field solver enforces the sampled infinitesimal preservation equation directly; it does not optimize the neural objective in \eqref{eq:objective-weighted}.

Embed each recovered field as $H=\left[\begin{smallmatrix}A&b\\0&0\end{smallmatrix}\right]$. The generated matrices act linearly on $(\theta,1)$, so the finite composition law is supplied by matrix multiplication. Preservation is tested on $128$ independently sampled parameter/input pairs per seed. Each group sample is a product of four exponentials with independently sampled coefficients in $[-0.18,0.18]$. The output norm is divided by a fixed root-mean-square output norm on these test points. Composition and inverse errors are relative to physical motion. This experiment measures extraction of preserving generators within the affine class.

The solver recovers four- and eight-dimensional hidden-basis algebras in the two- and three-layer linear networks and three-, six-, and eight-dimensional rescaling algebras in the ReLU networks (Table~\ref{tab:frontier-higher-recovery}). All three extraction seeds identify the corresponding reference span. Finite transformations of the floating-point estimates preserve the outputs to approximately $10^{-15}$ in float64, with median relative parameter changes of $4.2\%$--$6.4\%$.

\begin{table}[ht]
\centering\small
\setlength{\tabcolsep}{5pt}
\begin{tabular}{llrrrr}
\toprule
Network & Widths & Dimension & Center & Derived & Output error \\
\midrule
Linear & $(2,2,2)$ & 4 & 1 & 3 & $1.10\times10^{-15}$ \\
Linear & $(2,2,2,2)$ & 8 & 2 & 6 & $1.47\times10^{-15}$ \\
ReLU & $(2,3,2)$ & 3 & 3 & 0 & $1.13\times10^{-15}$ \\
ReLU & $(2,3,3,1)$ & 6 & 6 & 0 & $1.16\times10^{-15}$ \\
ReLU & $(2,2,2,2,2,1)$ & 8 & 8 & 0 & $1.26\times10^{-15}$ \\
\bottomrule
\end{tabular}
\caption{Higher-dimensional recovery by the affine-field solver. Widths include input and output dimensions. The dimension is inferred from the sampled constraints; all three extraction seeds agree. The center consists of generators commuting with every generator, and the derived algebra is the span of their commutators. Exact reconstruction confirms the reference-span identification and the displayed algebra dimensions. Output errors are medians across seeds of within-seed $95$th percentiles over $128$ independent parameter/input samples and four-factor transformations.}
\label{tab:frontier-higher-recovery}
\end{table}

\paragraph{Exact reconstruction and interpretation.} We use pivoted QR to choose coordinates for the recovered span and reconstruct its basis with rational coefficients of denominator at most $32$. These reconstructed matrices are distinct from the floating-point estimates. Exact arithmetic verifies their linear independence, preservation, equality with the reference span, and bracket closure. All $15$ reconstructed families satisfy these exact checks. For linear networks, every component of $D_\theta F_X(\theta)[A\theta+b]$ vanishes as a polynomial in the weights. For ReLU networks, the reconstructed fields are exact linear combinations of hidden-unit rescalings. Positive homogeneity then gives preservation at all inputs and parameters, including activation boundaries.

For example, in the three-layer linear network the reconstructed span agrees with the generators of
\[
 (W_1,W_2,W_3)\longmapsto
 (G_1W_1,\;G_2W_2G_1^{-1},\;W_3G_2^{-1}),
 \qquad G_1,G_2\in\mathrm{GL}^{+}(2).
\]
The product of the three matrices is unchanged. This identifies the recovered algebra as $\mathfrak{gl}(2)\oplus\mathfrak{gl}(2)$, with dimension eight, center dimension two, and derived dimension six. In the five-layer ReLU network, the eight hidden-unit rescalings give an eight-dimensional abelian algebra. These identifications concern the recovered reference families; the sampled nullspace calculation alone is not a classification of every symmetry of the architecture.

\begin{figure}[ht]
\centering
\begin{minipage}{0.485\linewidth}\centering
\includegraphics[width=\linewidth]{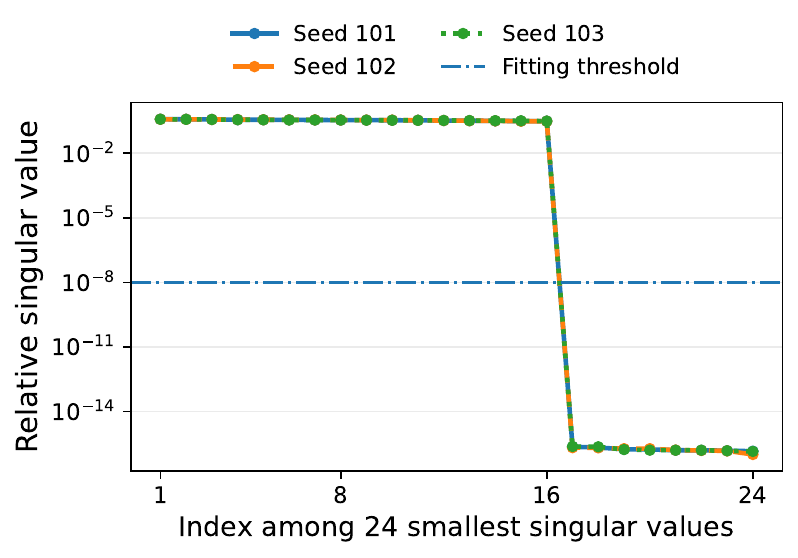}
\small (a) Three-layer linear network
\end{minipage}\hfill
\begin{minipage}{0.485\linewidth}\centering
\includegraphics[width=\linewidth]{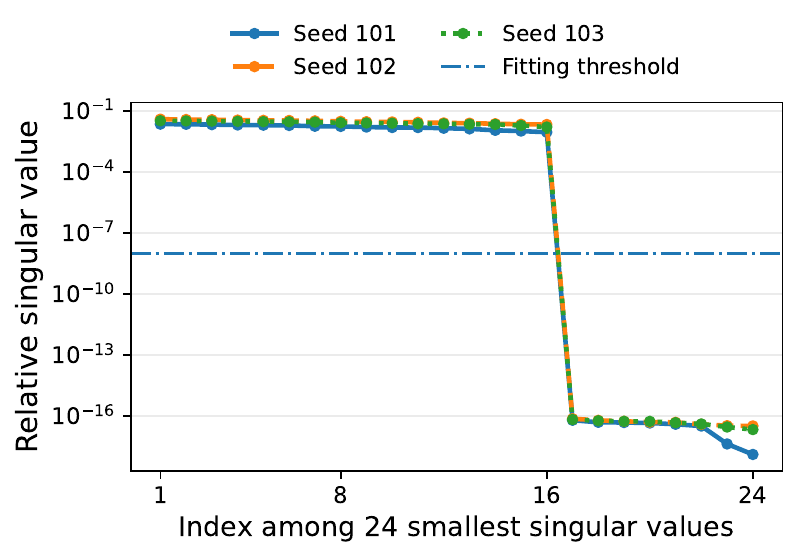}
\small (b) Five-layer ReLU network
\end{minipage}
\caption{Singular values of the stacked affine-field constraints. The horizontal axis orders the $24$ smallest singular values from largest to smallest; each vertical value is divided by the largest singular value of the complete constraint matrix. The fitting-threshold line is $10^{-8}$, below which a direction is retained. Seed labels identify three independently sampled extraction datasets. Each has eight retained directions in both networks, separated by a large gap. Table~\ref{tab:frontier-higher-recovery} gives their different bracket structures.}
\label{fig:frontier-spectra}
\end{figure}

\FloatBarrier

\section{Subnetwork discovery and lifting}
\label{app:experimental-subnetwork}

\paragraph{Paired host evaluation.} Local and full discovery are evaluated on the same $128$ complete host parameter vectors and inputs per seed, with identical coefficient draws and two factors per group sample. Their supports contain $10$ and $88$ parameters. Table~\ref{tab:experimental-paired} reports raw displacements and absolute host-output errors. At radius $0.5$, the per-seed error ratios favor local discovery by $4.7$--$9.9$ times. The preservation targets and supports differ, so this experiment measures the combined subnetwork strategy.

\begin{table}[ht]
\centering\small
\begin{tabular}{lrrrr}
\toprule
Radius & Local motion & Full motion & Local output & Full output \\
\midrule
0.1 & 0.0119 & 0.0115 & $1.92\!\times\!10^{-4}$ & $1.12\!\times\!10^{-3}$ \\
0.3 & 0.0430 & 0.0416 & $5.91\!\times\!10^{-4}$ & $4.09\!\times\!10^{-3}$ \\
0.5 & 0.0486 & 0.0472 & $9.12\!\times\!10^{-4}$ & $5.26\!\times\!10^{-3}$ \\
0.8 & 0.0853 & 0.0828 & $1.61\!\times\!10^{-3}$ & $8.20\!\times\!10^{-3}$ \\
\bottomrule
\end{tabular}
\caption{Paired local and full-parameter discovery in the $(2,8,8,1)$ ReLU host. Motion is the median across seeds of within-seed median raw displacement. Output is the median across seeds of within-seed $95$th-percentile absolute host error.}
\label{tab:experimental-paired}
\end{table}

\FloatBarrier

\section{Subnetwork actions in pretrained language models}
\label{app:experimental-pythia-details}

We fit a separate action for each fixed prefix from the training split of WikiText-103. The reported sites are Pythia-160M layer $6$ with eight tokens and Pythia-1B layer $3$ with $16$ tokens, using zero-based layer indices. Each result includes all three seeds $101$--$103$. Nearby-parameter tests and original-checkpoint installation are separate evaluations.

\subsection{Parameterization and training}

Let $\widetilde X$ contain the block inputs at all protected positions, augmented with a bias row. The selected boundary contribution is $U_BZ_B+U_C\sigma(\widetilde V_C\widetilde X)$, where $C$ contains one moving unit and $B$ contains $k$ compensating units. We hold $\widetilde V_B,U_C$, all other units, and the output bias fixed. The compensator feature matrix $Z_B$ is chosen by pivoted QR and remains fixed. The complete augmented input has rank $k$, and the accepted $Z_B$ has condition number at most $10^4$.

The action uses coordinates $(D,W)$ defined by
\[
    \widetilde X=QR,\qquad U_C=OR_U,\qquad
    \widetilde V_C=\widetilde V_{C,0}+DQ^\top,\qquad
    U_B=U_{B,0}+OW,
\]
where $Q$ and $O$ have orthonormal columns. With $P_0=\widetilde V_{C,0}\widetilde X$, we train on the reduced output
\[
    f_X(D,W)=WZ_B+R_U\sigma(P_0+DR).
\]
Here $f_X$ is the output target $L(\cdot,X)$ used in the discovery losses. Boundary-output changes equal $O\Delta f_X$ and satisfy $\|O\Delta f_X\|_F=\|\Delta f_X\|_F$. Likewise, $\|\Delta DQ^\top\|_F^2+\|O\Delta W\|_F^2=\|\Delta D\|_F^2+\|\Delta W\|_F^2$. The $2k$ coordinates thus retain the physical output and motion norms while restricting incoming changes to directions visible on the prefix. Each action call receives both current coordinates.

We use the width-$256$, three-hidden-layer SiLU decoder, exact identity, and one learned Frobenius-normalized $2\times2$ generator. Training uses float32 Adam, batch size $32$, gradient clipping at $10$, and no sparsity. The sampled state is the original checkpoint with probability $0.2$ and a Gaussian coordinate perturbation otherwise, with standard deviation $0.2$ times the original selected-weight root-mean-square magnitude. Original and transformed residual locations are mixed equally, and one quarter of composition pairs are inverse pairs. The training radius is $0.5$ with up to three signed exponential factors.

For each objective, three recipes train for $5{,}000$ steps and the strongest two continue to $15{,}000$. Validation selects a recipe for three fresh-seed fits of $15{,}000$ steps. Base recipes use learning rates $10^{-3}$ or $3\times10^{-4}$, preservation weight $10$, composition weight $1$, and hybrid transport weight $1$. The stronger recipe uses learning rate $3\times10^{-4}$, preservation weight $30$, composition weight $3$, and hybrid transport weight $3$ plus finite preservation weight $30$. Scale weight is $1$. Final tests use an independent random stream. Objectives are tuned separately at equal step and search budgets.

\subsection{Nearby-parameter validation}

Local tests use $512$ samples per radius/factor cell and the one-generator numerical criterion in Appendix~\ref{app:experimental-protocol}. The fixed parameter scale is the original selected-weight root-mean-square magnitude; the output scale is the norm of the original moving unit's contribution on the prefix. Table~\ref{tab:pythia-local} gives the three-seed Pythia-160M eight-token results.

\begin{table}[ht]
\centering\small
\setlength{\tabcolsep}{4.5pt}
\begin{tabular}{lrrrrrr}
\toprule
Objective & Output & Composition & Inverse & Transport & Cancellation & Fits \\
\midrule
Finite & $6.00\!\times\!10^{-4}$ & $8.00\!\times\!10^{-3}$ & $3.16\!\times\!10^{-3}$ & $1.40\!\times\!10^{-2}$ & $2.43\!\times\!10^{-2}$ & 0/3 \\
Hybrid & $2.18\!\times\!10^{-4}$ & $1.93\!\times\!10^{-3}$ & $1.12\!\times\!10^{-3}$ & $1.72\!\times\!10^{-3}$ & $9.98\!\times\!10^{-3}$ & 2/3 \\
\bottomrule
\end{tabular}
\caption{Pythia-160M, layer $6$, eight protected tokens. Errors are medians across three seeds of within-seed $95$th percentiles at radius $0.5$ and one factor. The final column counts fits satisfying every tolerance, including two factors, subdivision, and visible compensation.}
\label{tab:pythia-local}
\end{table}

\subsection{Installed weights and full-model outputs}

Installation begins at the original checkpoint. We change the selected incoming row, bias, and compensating output columns, cast to float32, run the complete model, and restore the weights. Logit measurements test how preservation of the trained boundary output carries through the remaining computation. The prefix stays fixed within fitting and action-law tests.

For each coefficient radius $\rho\in\{0.1,0.3,0.5,0.8\}$, we evaluate $g=\exp(th)$ with $\|h\|_F=1$ at $16$ signed coefficients. Eight magnitudes are sampled uniformly on $[0.2\rho,\rho]$ and paired with their negatives. The same normalized magnitudes $|t|/\rho$ are reused across radii, and methods share coefficient draws within each seed. Thus $\rho$ bounds the exponential coefficient, not the physical weight displacement. Let $Z$ and $Z'$ denote the original and changed logits on the protected positions. The reported relative error is $\|Z'-Z\|_F/(\|Z\|_F+10^{-12})$, and the absolute root-mean-square error is $\|Z'-Z\|_F/\sqrt{N}$ for $N$ logit entries. Cancellation compares the sum of the moving and compensating output changes with the sum of their norms. Motion is measured on the selected weights after casting. Selected-weight change is $100\|\theta_S^{\prime}-\theta_S\|_2/\|\theta_S\|_2$, computed on the raw selected weights $S$ at the original checkpoint. Composition and inverse tests round after each action and retain the resulting stored weights for the next application. Their relative errors use the same two-step and one-step displacement denominators defined in Appendix~\ref{app:experimental-protocol}. TF32 is disabled.

The analytic comparator translates the incoming coordinates and solves for compensating output coefficients. Its direction and motion scale are fixed on validation data, as are those of the random additive translation. Incoming-only retains the learned incoming update and omits compensation, giving less total motion. The learned decoder receives neither compensation labels nor a solve.

To examine data dependence, we evaluate the same frozen edits on the protected prefix, its teacher-forced continuation, and three reserved test-article windows. In the 1B example, the learned relative-logit error tails are $1.50\times10^{-6}$, $2.12\times10^{-5}$, and $2.24\times10^{-5}$, respectively. The fresh-text statistic averages the three article errors for each edit before computing its within-seed percentile. These measurements describe preservation on the fitted prefix and behavior of the same changed weights on unprotected inputs. All full-model results in this subsection concern the original checkpoint.

\begin{table}[ht]
\centering\small
\setlength{\tabcolsep}{4.5pt}
\begin{tabular}{llrrrr}
\toprule
Model & Transformation & Motion (\%) & Relative logits & RMS logits & Cancellation \\
\midrule
160M & Learned & 0.1541 & $8.82\!\times\!10^{-7}$ & $7.13\!\times\!10^{-4}$ & $4.54\!\times\!10^{-3}$ \\
160M & Analytic & 0.1523 & $8.36\!\times\!10^{-7}$ & $6.76\!\times\!10^{-4}$ & $1.95\!\times\!10^{-5}$ \\
160M & Incoming only & 0.0231 & $1.84\!\times\!10^{-6}$ & $1.48\!\times\!10^{-3}$ & $1.00\!\times\!10^{0}$ \\
160M & Random & 0.1520 & $3.09\!\times\!10^{-6}$ & $2.50\!\times\!10^{-3}$ & $6.56\!\times\!10^{-1}$ \\
\midrule
1B & Learned & 0.0687 & $1.50\!\times\!10^{-6}$ & $5.37\!\times\!10^{-6}$ & $3.29\!\times\!10^{-3}$ \\
1B & Analytic & 0.0695 & $1.41\!\times\!10^{-6}$ & $5.02\!\times\!10^{-6}$ & $5.22\!\times\!10^{-4}$ \\
1B & Incoming only & 0.0105 & $7.35\!\times\!10^{-5}$ & $2.62\!\times\!10^{-4}$ & $1.00\!\times\!10^{0}$ \\
1B & Random & 0.0688 & $1.93\!\times\!10^{-4}$ & $6.90\!\times\!10^{-4}$ & $8.70\!\times\!10^{-1}$ \\
\bottomrule
\end{tabular}
\caption{Installed transformations at radius $0.5$, starting from the original checkpoint. Motion is the median percentage change in the selected parameter block. Error columns are medians across all three seeds of within-seed $95$th percentiles over $16$ transformations. Analytic and random controls are calibrated to similar physical motion; incoming-only removes the learned compensation.}
\label{tab:pythia-installed}
\end{table}

Figure~\ref{fig:pythia160-controls} gives the Pythia-160M counterpart to Figure~\ref{fig:pythia-controls}. At layer $6$ with eight protected tokens, learned compensation keeps protected-logit drift near the analytic reference across the displayed coefficient range. At radius $0.5$, hybrid training reduces stored-weight composition error from $2.51\times10^{-3}$ to $3.08\times10^{-4}$, with similar raw displacement. These plots evaluate the installed edits at the original checkpoint; Table~\ref{tab:pythia-local} separately describes nearby-parameter tests.

\begin{figure}[!htbp]
\centering
\begin{minipage}{0.485\linewidth}\centering
\includegraphics[width=\linewidth]{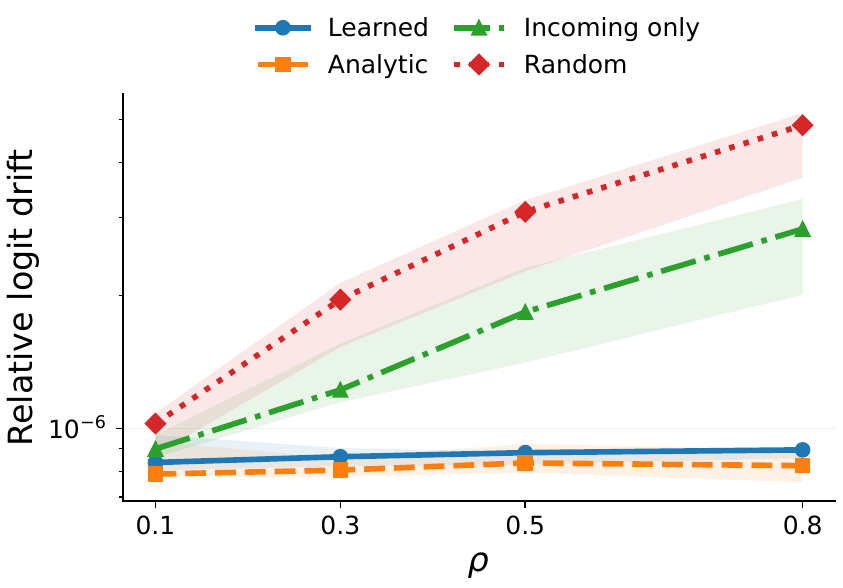}
\small (a) Preserving protected outputs
\end{minipage}\hfill
\begin{minipage}{0.485\linewidth}\centering
\includegraphics[width=\linewidth]{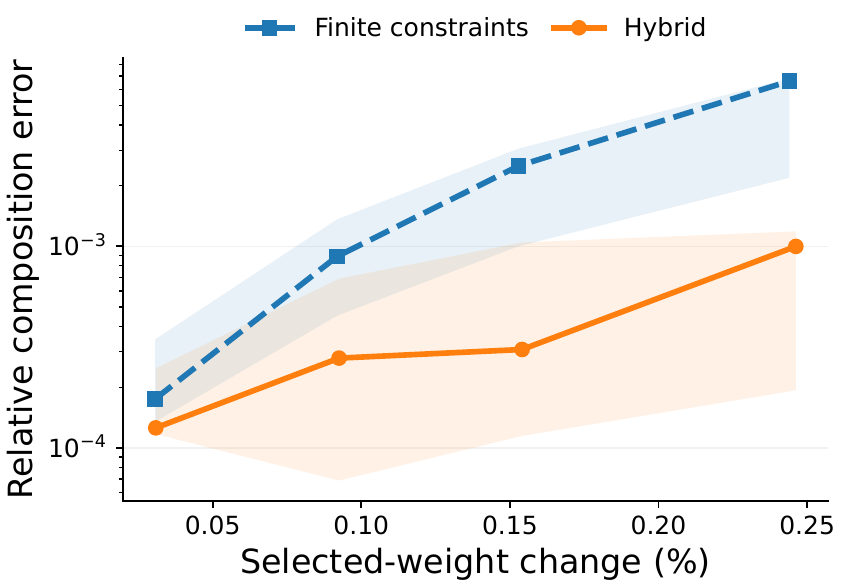}
\small (b) Composing learned transformations
\end{minipage}
\caption{Pythia-160M, layer $6$, at the original checkpoint and one fixed eight-token prefix. (a) Coefficient radius $\rho$ bounds $|t|$ in $g=\exp(th)$; magnitudes lie in $[0.2\rho,\rho]$. Relative logit drift is $\|Z'-Z\|_F/\|Z\|_F$ on the prefix. Learned denotes the hybrid action; analytic solves for compensation; random is a motion-calibrated translation; incoming-only removes the learned compensation. (b) Selected-weight change is $100\|\theta'_S-\theta_S\|_2/\|\theta_S\|_2$ on edited parameters $S$. Relative composition error compares sequential and combined edits, divided by their total two-step displacement, with rounding after each update. Finite constraints and hybrid use separately tuned objectives. Each curve aggregates within-seed statistics across all three seeds by their median; error statistics are $95$th percentiles over $16$ signed transformations per radius, and motion statistics are medians. Shading spans the seed statistics.}
\label{fig:pythia160-controls}
\end{figure}

The analytic reference has a measured finite-precision logit error of $8.36\times10^{-7}$ in the 160M example and $1.41\times10^{-6}$ in the 1B example. The learned transformations reach similar protected-logit errors, with larger cancellation residuals. Table~\ref{tab:pythia-stored} separately reports composition and inverse consistency of the stored-weight transformations. At the 1B site, hybrid training reduces the median stored composition error from $3.21\times10^{-3}$ to $5.81\times10^{-4}$ at nearly identical displacement.

\begin{table}[ht]
\centering\small
\begin{tabular}{llrrr}
\toprule
Model & Objective & Raw motion & Composition & Inverse \\
\midrule
160M & Finite & $2.94\!\times\!10^{-3}$ & $2.51\!\times\!10^{-3}$ & $1.91\!\times\!10^{-3}$ \\
160M & Hybrid & $2.97\!\times\!10^{-3}$ & $3.08\!\times\!10^{-4}$ & $5.92\!\times\!10^{-4}$ \\
1B & Finite & $2.39\!\times\!10^{-3}$ & $3.21\!\times\!10^{-3}$ & $2.86\!\times\!10^{-3}$ \\
1B & Hybrid & $2.39\!\times\!10^{-3}$ & $5.81\!\times\!10^{-4}$ & $9.00\!\times\!10^{-4}$ \\
\bottomrule
\end{tabular}
\caption{Stored-weight action consistency at radius $0.5$. Both learned objectives use the same selected support and separately tuned recipes. Error entries are medians across three seeds of within-seed $95$th percentiles. These evaluations start at the original checkpoint.}
\label{tab:pythia-stored}
\end{table}

\end{document}